\documentclass{article}
\usepackage{authblk}
\title{NEWMA: a new method for scalable model-free online change-point detection}
\author[1]{Nicolas Keriven}
\author[2]{Damien Garreau}
\author[3]{Iacopo Poli}
        
\date{}
        
\affil[1]{CNRS, GIPSA-lab.}
\affil[2]{Universit\'e C\^ote d'Azur, Inria, CNRS, LJAD.}
\affil[3]{LightOn.}

\usepackage{subcaption}

\usepackage{cite}

  \usepackage{graphicx}

\usepackage{amsmath}

\usepackage[noend]{algcompatible}

\usepackage{array}

\usepackage{url}

\usepackage[margin=1.2in]{geometry}

\usepackage{hyperref}
\usepackage{amssymb,amsthm}
\usepackage{mathtools}
\usepackage{color}
\usepackage{enumitem}
\usepackage{xspace}
\usepackage{stmaryrd}

\usepackage{tikz}
\usetikzlibrary{arrows}
\usetikzlibrary{shapes}
\usetikzlibrary{calc}
\usetikzlibrary{decorations.pathreplacing}

\newcommand{\rev}[1]{#1}

\newtheorem{theorem}{Theorem}

\newtheorem{lemma}[theorem]{Lemma}
\newtheorem{proposition}[theorem]{Proposition}

\newtheorem{remark}{Remark}

\newcommand{\paren}[1]{\left(#1\right)}
\newcommand{\brac}[1]{\left[#1\right]}
\newcommand{\inner}[1]{\left\langle#1\right\rangle}
\newcommand{\norm}[1]{\left\|#1\right\|}
\newcommand{\set}[1]{\left\{#1\right\}}
\newcommand{\abs}[1]{\left\lvert #1 \right\rvert}

\newcommand{\order}[1]{\mathcal{O}\paren{#1}}

\def \Exp {\mathbb{E}}
\def\PP{\mathbb{P}}

\def \Prob {\pi}

\newcommand{\diag}[1]{\textrm{diag}\paren{#1}}
\newcommand*\diff{\mathop{}\!\mathrm{d}}

\def \iid {\emph{i.i.d.}~}
\def \ie {\emph{i.e.}\@\xspace}
\def \eg {\emph{e.g.}\@\xspace}
\def \simiid {\overset{\iid}{\sim}}

\DeclareMathOperator*{\argmin}{arg\,min}

\def \RR {\mathbb{R}}
\def \CC {\mathbb{C}}
\def \NN {\mathbb{N}}

\def \Sketch {\mathbf{z}}

\def\sample{x}

\def \kernel {K}
\def \kernelPsi {\kappa}

\def \freqdist {\Gamma}
\def \freq {{\boldsymbol \omega}}

\def \feat {\phi}
\def \rfeat {{\feat_{\freq}}}

\def \pFail {\rho}

\def \ve {\mathbf{e}}
\def \vx {\mathbf{x}}

\def \vu {\mathbf{u}}
\def \vv {\mathbf{v}}

\def \mA {\mathbf{A}}
\def \mB {\mathbf{B}}

\def \mP {\mathbf{P}}

\def \AnyOpNA{\Psi}
\newcommand{\AnyOp}[1]{\Psi\paren{#1}}

\def \Tfalsealarm {\overline{T}}

\newcommand{\enet}{\mathcal{C}}

\def \Ball {\mathcal{B}}

\definecolor{darkpurple}{rgb}{0.3,0,0.3}

\def\updt{\Lambda}
\def\updtp{\lambda}

\def\thres{\tau}
\def\winsize{B}
\def\MeasSpace{\mathcal{H}}

\def\updtCst{c}
\def\MMD{\textrm{\textup{MMD}}}
\def\eigen{\xi}
\def\eigenfunc{\psi}
\newcommand{\kroneck}[1]{\textrm{\textup{I}}_{\set{#1}}}
\def\eqdef{\stackrel{\textrm{\textup{def.}}}{=}}
\def\Id{\textrm{\textup{Id}}}

\def\Efund{E_\textup{conc.}}
\def\Einit{E_\textup{init.}}
\def\epsfund{\varepsilon_1}
\def\epsinit{\varepsilon_2}
\def\Ewin{E_\textup{assum.}}
\def \distphi{d_\AnyOpNA}
\def\EWMA{{EWMA}}
\def\NEWMA{{NEWMA}}

\newcommand{\edit}[1]{{#1}}

\begin{document}

\maketitle

\begin{abstract}
We consider the problem of detecting abrupt changes in the distribution of a multi-dimensional time series, with limited computing power and memory. In this paper, we propose a new, simple method for model-free online change-point detection that relies only on fast and light recursive statistics, inspired by the classical Exponential Weighted Moving Average algorithm (EWMA). The proposed idea is to compute \emph{two} EWMA statistics on the stream of data with different forgetting factors, and to compare them. By doing so, we show that we implicitly compare recent samples with older ones, without the need to explicitly store them. Additionally, we leverage Random Features (RFs) to efficiently use the Maximum Mean Discrepancy as a distance between distributions, furthermore exploiting recent optical hardware to compute high-dimensional RFs in near constant time. We show that our method is significantly faster than usual non-parametric methods for a given accuracy. 
\end{abstract}

\section{Introduction}

The goal of online change-point detection is to detect abrupt changes in the distribution of samples in a data stream. 
One seeks to detect a change as soon as it occurs, while minimizing the number of false alarms.
Online change-point detection has numerous practical applications, for instance medical monitoring via the segmentation of EEG, ECG and fMRI signals~\cite{Malladi2013,Staudacher2005,Bosc2003}, or detections of changes in audio~\cite{Bie_Bac_Con:2015} or video~\cite{Kim_Mar_Per:2009,Abo_Gou_Blo:2015} streams.  
We refer to~\cite{Pol_Tar:2012} for a thorough review. 
In recent applications, the need arises to perform such methods on embedded devices, for instance in video streams from body-worn or surveillance video cameras~\cite{Allen2016}, or on data collected by smart phones~\cite{Khan2016a}. 
In addition to being constrained by limited power and memory, such personal devices collect data that can be potentially sensitive, hence the need to process the stream on-the-fly, ideally without storing any raw data. 

\edit{In this paper, we propose a new approach for online, non-parametric change-point detection, whose main advantage is that it does not require to store any raw data in memory, but only appropriate smoothed quantities. It is inspired by: a) the classical Exponentially-Weighted Moving Average (EWMA), but requires less prior knowledge about the in-control distribution of the data, and b) a simple Sliding Window (SW) strategy in its model-free version, but is more efficient in memory and preserves data privacy.}


\subsection{Framework: model-free methods and generalized moments}

We consider a stream of samples $\left(\sample_t\right)_{t\in\NN}$ with values in $\RR^d$ with potentially large $d$. 
The goal of online change-point detection is to detect changes in the distribution of the samples $x_t$ in a sequential manner. \rev{We assume that the samples are independent and identically distributed ($i.i.d.$) before and after each change, and that there may be multiple changes in a row to be detected on-the-fly. As we will see, some methods assume prior knowledge about the distributions before (and sometimes after) each change, however we will consider that no prior knowledge is available here, and develop a so-called \emph{model-free} method.}

Historically, many methods assume that the distributions \rev{of interest} belong to a parametric family of distributions whose likelihood $p_\theta$ is entirely specified (often Gaussians), and rely on a (generalized) likelihood ratio \rev{(GLR)} test. However, such a complete specification is not always available. \rev{In a non-parametric context, some methods then rely on approximating generic discrepancies between distributions such as the Kullback-Leibler (KL) divergence, the total variation~\cite{Kifer2004} or some optimal transport-based distances~\cite{Cheng2019}. However, it is well-known \cite{Weed2017} that most of these metrics are difficult to estimate in high dimension $d$, and/or may be too costly to compute in an online framework.} \rev{On the contrary, simpler} methods are designed to detect changes in some quantity related to the distribution such as the mean or the variance~\cite{Costa2006}. We consider a generalization of this last concept, namely, to detect changes in \textbf{a collection of generalized moments} $\theta_\AnyOpNA(\Prob) \eqdef \Exp_{\sample\sim\Prob} \AnyOp{\sample}$, where $\Prob$ is the distribution of the samples, and $\AnyOpNA:\RR^d \to \MeasSpace$ is a mapping to a normed space $(\MeasSpace,\norm{\cdot})$ (generally, $\MeasSpace = \RR^m$ or $\CC^m$). We therefore introduce the following pseudometric on distributions:
\begin{equation}\label{eq:distphi}
\distphi(\Prob,\Prob') \eqdef \norm{\Exp_{\Prob}\AnyOp{\sample}-\Exp_{\Prob'}\AnyOp{\sample}}
\, ,
\end{equation}
which measures how different two distributions are in terms of these moments.
For instance, when $\AnyOpNA =\Id$, then $\theta_\AnyOpNA(\Prob) =\Exp x$, \edit{and the underlying assumption is that changes will occur in the mean of the distribution of the samples.} 
This also includes higher order moments ($\Psi(x) = x^{\otimes k}$) or histograms ($\Psi(x) = (1_{x\in B_i})_{i=1}^k$ where the $B_i$ are regions of space). If infinite-dimensional spaces $\MeasSpace$ such as Reproducing Kernel Hilbert Spaces (RKHS) are considered, this framework also includes the so-called \emph{kernel} change-point detection~\cite{Harchaoui2009, Gar_Arl:2018}, and $\distphi$ is then referred to as the Maximum Mean Discrepancy (MMD)~\cite{Gretton2007}. We note that this framework does not, strictly speaking, include \emph{centered} moments such as the variance, however one could modify the definition of $\distphi$ to compute the variance from first order and second-order moments. We do not consider centered moments here for simplicity. 

If the user has \emph{prior knowledge} about which quantity $\theta_\AnyOpNA(\Prob)$ is susceptible to change over time, then $\AnyOpNA$ can be chosen accordingly. If not, we will see in Sec.~\ref{sec:MMD} that a somewhat ``universal'' embedding can be obtained by taking $\AnyOpNA$ as \emph{kernel random features} \cite{Rahimi2007}, \rev{which allows to efficiently approximate the MMD with high probability and controlled memory resources.} 

\subsection{Prior knowledge on the in-control statistic} 



\edit{\rev{As mentioned above, some} methods assume prior knowledge about the in-control distribution, that is, the distribution before the change. In our framework, it corresponds to the knowledge of the generalized moments $\theta^\star = \theta_\AnyOpNA(\Prob^\star)$, where $\Prob^\star$ is the in-control distribution.}

One such classical approach is the Exponential Weighted Moving Average (EWMA) algorithm~\cite{Roberts1959}, which we describe\footnote{Note that our description of EWMA is similar to the original \cite{Roberts1959}, with the addition that the data are transformed by the mapping $\AnyOpNA$.} in Alg.~\ref{alg:ewma}. 
EWMA computes \emph{recursively} a weighted average of $\AnyOp{\sample_t}$, with exponential weights that favor the more recent samples:
\[
\Sketch_t = (1-\updt) \Sketch_{t-1} + \updt  \AnyOp{\sample_t}
\]
where $0< \updt < 1$. When this average deviates too much from $\theta^\star$, an alarm is raised. The exponential weights (instead of, say, uniform weights) reduce the detection delay, and increase robustness to potentially irrelevant data in the past. 

\edit{When $\MeasSpace = \RR^m$, a classical multivariate extension of EWMA is called Multivariate-EWMA, and rely on the fact that every dimension may not need the same forgetting factor, and therefore replace $\updt$ by a diagonal matrix $\diag{\updt_1,\ldots,\updt_m}$. In our case there is no assumption on the marginals of the distribution of the data, and moreover it is not clear how the presence of the mapping $\AnyOpNA$ would affect this strategy, hence we consider a single forgetting factor instead, and note however that our method could be extended when $\updt$ is a matrix, which we leave for future work.}

From our point of view, when considering high-dimensional data, the main advantage of \EWMA{} is that it is extremely fast and have low memory footprint, due to its recursive nature: when a new sample arrives, the cost of the update is essentially that of computing $\AnyOp{\sample}$ once. 
Moreover, it preserves data privacy, in the sense that it never stores raw samples but only a smoothed statistic computed from them. 
However, EWMA requires the prior knowledge of $\theta^\star$, which severely limits its use in some cases where it is not available. 

\begin{figure}
\begin{algorithmic}
 \REQUIRE Stream of data $\sample_t$, function $\AnyOpNA$, in-control value $\theta^\star$, forgetting factor $0<\updt<1$, threshold $\thres>0$, initial value $\Sketch_0$
\FOR{$t=1,2,\ldots$}
\STATE$\Sketch_t = (1-\updt) \Sketch_{t-1} + \updt  \AnyOp{\sample_t}$\;
\IF{$\norm{\Sketch_t - \theta^\star} \geq \thres$}
\STATE Flag $t$ as a change-point
\ENDIF
\ENDFOR
\end{algorithmic}
\caption{\EWMA{}~\cite{Roberts1959}}
\label{alg:ewma}
\end{figure}

\subsection{Methods without prior knowledge} 

To solve this last problem, methods with no prior knowledge requirement about the in-control distribution were proposed. Many of them are \emph{two-steps} adaptation of the previous class of approaches: \edit{the parameter $\theta^\star$ is estimated from some training samples during a \emph{Phase I}, before the actual detection during a \emph{Phase II}~\cite{Hawkins2010, Zou2010a}. In the \emph{online} setting, where several changes can happen during a continuous run, this strategy is often adapted in a ``sliding windows'' approach: a window of recent samples is compared against a window of samples that came immediately before~\cite{Kifer2004,Liu2013a,Li2015b}.  
In our settings, given a window size $\winsize$, the most natural approach is to compare an \emph{empirical average} of $\AnyOp{x}$ over the last $\winsize$ samples with one computed on the $\winsize$ samples that came before, to approximate $\distphi$. When the difference is higher than a threshold, an alarm is raised. We refer to this simple algorithm as Sliding Window\footnote{While this simple algorithm appears several times in the literature~\cite{Kifer2004, Li2015b}, as far as we know it does not have a designated name.} (SW, Alg.~\ref{alg:ma}).}

%

Such \emph{model-free} methods are useful in a wide class of problems, since they can adapt to potentially any in-control situation. 
%
Despite these advantages, they can have a high memory footprint, since they store raw data that may be high-dimensional (see Tab.~\ref{tab:complexities} in Sec.~\ref{sec:newma}). 

\begin{figure}
\begin{algorithmic}
\REQUIRE Stream of data $\sample_t$, function $\AnyOpNA$, in-control value $\theta^\star$, forgetting factor $0<\updt<1$, threshold $\thres>0$
\STATE Initialize $\Sketch_{2B} =\frac{1}{B} \sum_{i=1}^B \AnyOp{x_i}$, $\Sketch'_{2B} =\frac{1}{B} \sum_{i=1}^{B} \AnyOp{x_{B+i}}$
\FOR{$t=2B+1,\ldots$}
\STATE $\Sketch_t = \Sketch_{t-1} + \frac{1}{B} (\AnyOp{\sample_{t-B}}- \AnyOp{\sample_{t-2B}})$
\STATE $\Sketch_t' = \Sketch_{t-1}' + \frac{1}{B} (\AnyOp{\sample_t}- \AnyOp{\sample_{t-B}})$
\IF{$\norm{\Sketch_t - \Sketch'_t} \geq \thres$}
\STATE Flag $t$ as a change-point
\ENDIF
\ENDFOR
\end{algorithmic}
\caption{Sliding Window (SW) (\eg,~\cite{Kifer2004})}
\label{alg:ma}
\end{figure}

\subsection{Contributions and outline of the paper} 

The main goal of this paper is to propose a method that gets the best of both worlds, that is, that does not store any raw data, like EWMA, while being simultaneously free of prior knowledge like SW. To this end, in Sec.~\ref{sec:newma}, we introduce ``No-prior-knowledge'' EWMA (NEWMA), based on a simple and intuitive idea: compute \textbf{two} EWMA statistics with \emph{different} forgetting factors, and flag a change when the distance between them crosses a threshold. 
We show that NEWMA mimics the behavior of the SW algorithm by implicitly comparing pools of recent and old samples, but \emph{without having to keep them in memory}. 
In Sec.~\ref{sec:MMD}, we show how choosing $\AnyOpNA$ as Random Features (RFs)~\cite{Rahimi2007} brings the method closer to kernel change-point detection~\cite{Harchaoui2009}, and in particular its online version the so-called Scan-$B$ algorithm~\cite{Li2015b}, while retaining low complexity and memory footprint.
In Section~\ref{sec:thres}, we examine how to set the detection threshold. We first review two classical ``parametric'' approaches, which are however generally not applicable in practice in model-free situations, then propose a numerical procedure for computing on-the-fly a dynamic threshold $\thres$, which empirically performs better than a fixed threshold. 
Experiments over synthetic and real data are presented in Sec.~\ref{sec:expe}, where we take advantage of a sublinear construction of RFs~\cite{Le2013} and, more strikingly, of a recent development in optical computing~\cite{Saade2016} that can compute RFs in $\order{1}$ for a wide range\footnote{The limitations are due to the optical hardware itself. Currently, state-of-the-art Optical Processing Units (OPU) can compute random features in constant time for $d$ and $m$ in the order of millions.} of dimensions $d$ and numbers of features $m$.
We show that our algorithm 
retrieves change-points at a given precision significantly faster than competing model-free approaches.

\section{Related Work}\label{sec:related}

As mentioned before, the idea of using several forgetting factors in recursive updates has been proposed in the so-called Multivariate EWMA (MEWMA)~\cite{Lowry1992,Khan2016a}, which uses a different factor for each coordinate of multivariate data, or to optimize the detection over different time-scales~\cite{Han2007}. It is different from NEWMA, which computes and compares two recursive averages over the same data. 
Closer to NEWMA, it has been pointed out to us that the idea of using several forgetting factors is used in a trading method called \emph{moving average crossover} (which, to the best of our knowledge, has never been published): it consists in computing two recursive averages over uni-dimensional data (such as pricing data), and interpreting the time when they ``cross'' (change relative position) as indicating a smoothed general trend of pricing going up or down. In addition to handling multi-dimensional data (which nullifies the meaning of going ``up'' or ``down''), NEWMA exploits these statistics in a very different way: it compute the \textbf{difference} between the two recursive averages, in order to extract time-varying information without keeping any sample in memory.
To the best of our knowledge, the key idea behind NEWMA has not been proposed before.

Dimension reduction methods such as \emph{sketching} have been used in the context of high-dimensional change-point detection~\cite{Xie2016}. In our notations, it corresponds to choosing a mapping~$\AnyOpNA$ which is dimension-reducing ($m\ll d$). While the authors in~\cite{Xie2016} then considers classical parametric methods in the new low-dimensional space, in a non-parametric context their approach could be combined with NEWMA for additional memory gain.

As described in Section \ref{sec:MMD}, when using RFs as the mapping~$\AnyOpNA$, our framework bears connection with the kernel change-point detection methodology~\cite{Harchaoui2009, Gar_Arl:2018}, in which the original estimator of the MMD based on a $U$-statistic is considered~\cite{Gretton2007} instead of averaged random features. In particular, an \emph{online} version of kernel change-point has been proposed in~\cite{Li2015b}, with the so-called Scan-$B$ algorithm. It is a variant of the sliding window approach, which however compares a window of recent samples with \emph{several} past windows, instead of only one as in SW.

Finally, the use of low-dimensional mappings or RFs have been proposed for fast anomaly detection (which is slightly different from change-point detection) in~\cite{Gopalan2018, Francis2018}, where the authors also describe how to exploit low-rank approximations to accelerate the method. In our paper, we show how NEWMA offers a different kind of acceleration in the context of change-point detection, especially when exploiting optical RFs~\cite{Saade2016}.


\section{Proposed algorithm}\label{sec:newma}



\begin{figure}
\begin{algorithmic}
\REQUIRE Stream of data $\sample_t$, function $\AnyOpNA$, forgetting factors $0<\updtp<\updt<1$, threshold $\thres>0$, initial value $\Sketch_0 = \Sketch'_0$
\FOR{$t=1,2,\ldots$}
\STATE $\Sketch_t = (1-\updt) \Sketch_{t-1} + \updt  \AnyOp{\sample_t}$
\STATE $\Sketch'_t = (1-\updtp) \Sketch'_{t-1} + \updtp  \AnyOp{\sample_t}$
\IF{$\norm{\Sketch_t - \Sketch'_t} \geq \thres$}
\STATE Flag $t$ as a change-point
\ENDIF
\ENDFOR
\end{algorithmic}
\caption{ \NEWMA{} (proposed)}
\label{alg:newma}
\end{figure}

\begin{figure}[h]
\centering
\includegraphics[width = 0.32\textwidth]{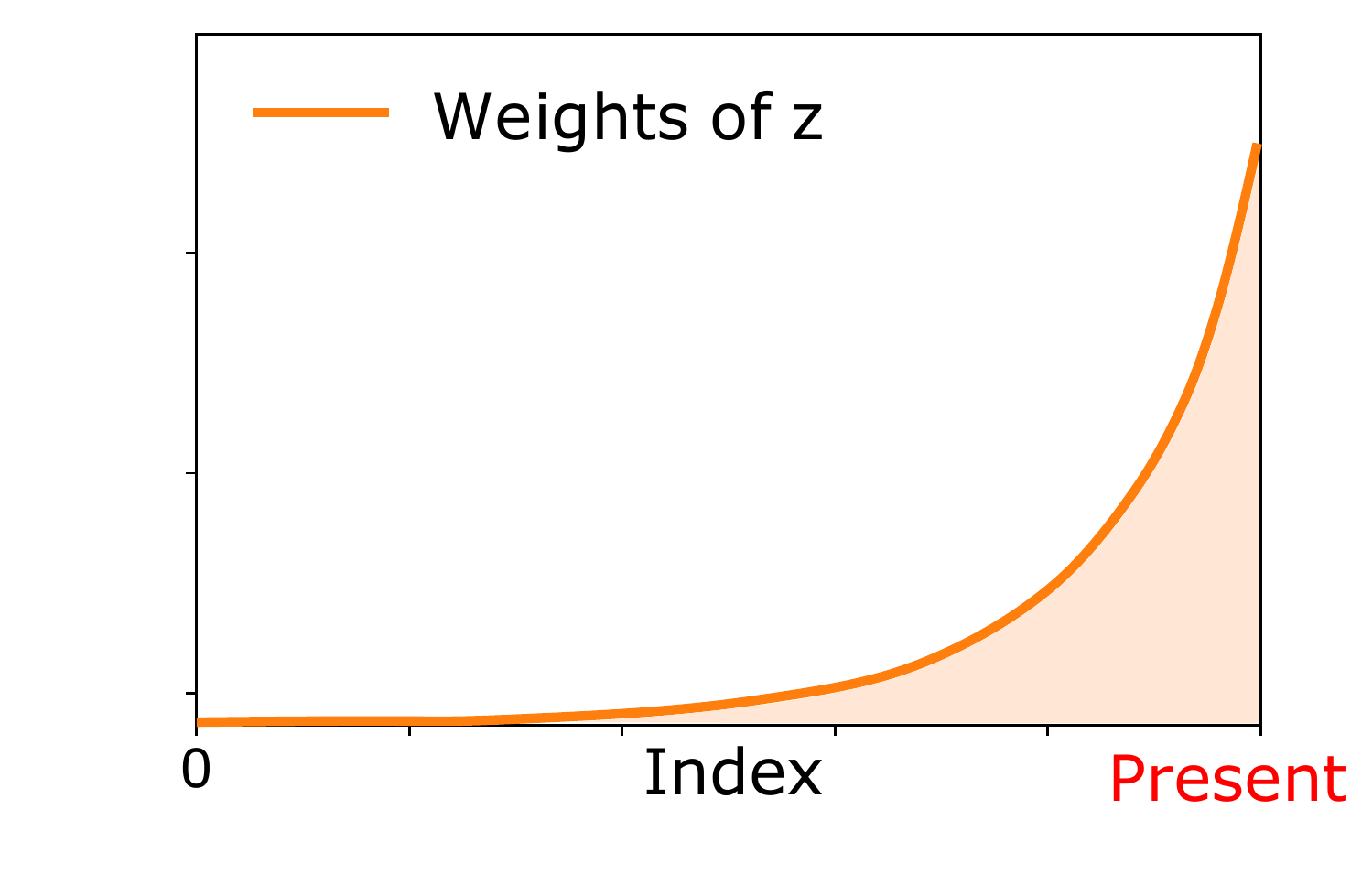}
\includegraphics[width = 0.32\textwidth]{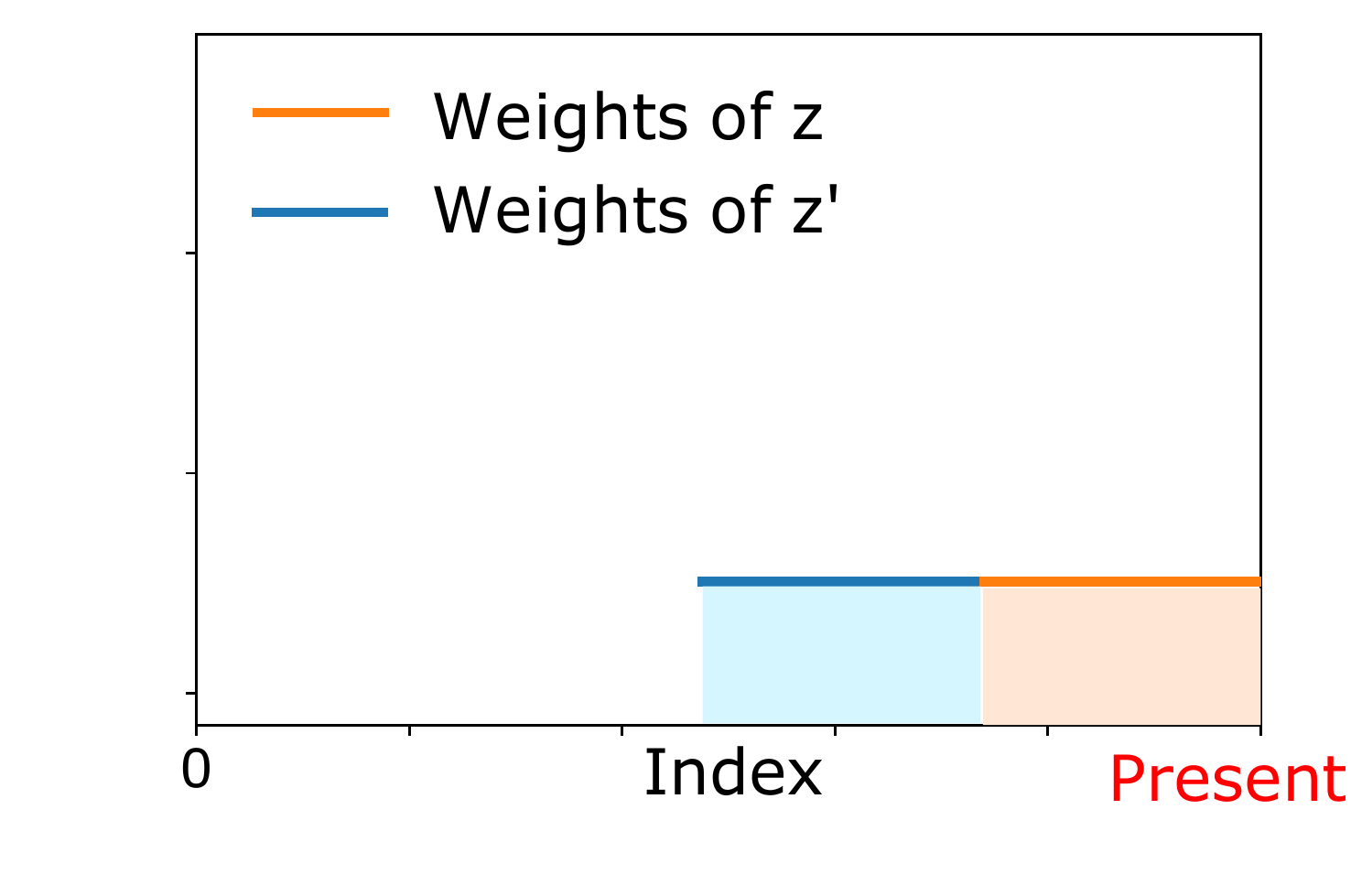}
\includegraphics[width = 0.32\textwidth]{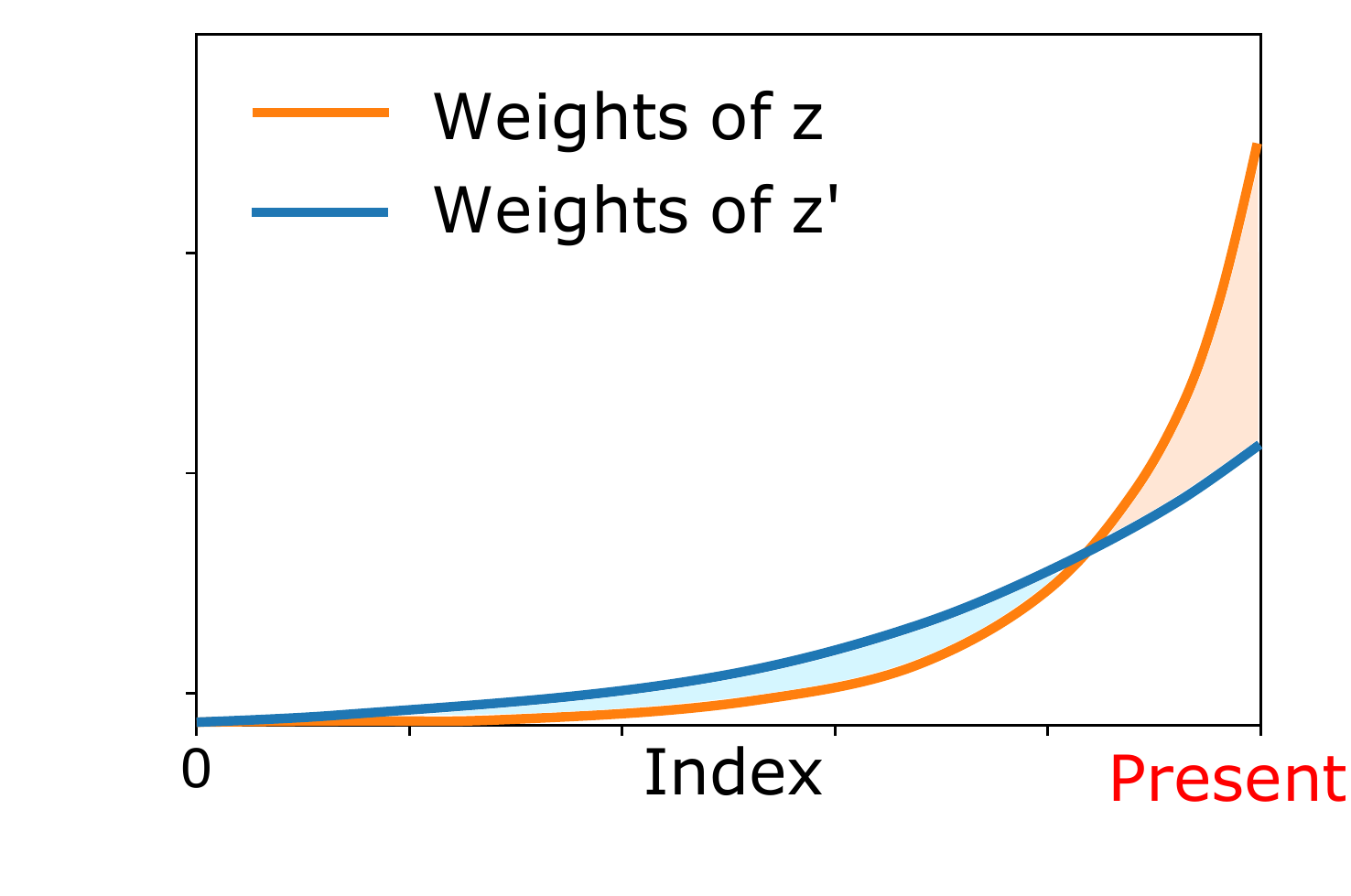}
\caption{\label{fig:illus_weights}Weights used in the empirical average computations in EWMA (top left), SW (top right), and NEWMA (bottom) algorithms as a function of time. 
In orange (resp. blue), the weights associated to the average $\Sketch_t$ (resp. $\Sketch'_t$). 
}
\end{figure}

\begin{table}[h]
\caption{\label{tab:complexities}Computational and memory footprint of the main algorithms discussed in this article. $C_\AnyOpNA$ (resp. $M_\AnyOpNA$) indicates time complexity (resp. the memory requirement) of computing $\AnyOpNA$ (see Sec. \ref{sec:RFcost}). In Scan-$B$, $N$ is the number of windows of size $\winsize$ considered.} 
\centering
\small
\begin{tabular}{lll}
\textbf{ALGORITHM} & \textbf{TIME} & \textbf{MEMORY} \\
\hline
\EWMA{} (Alg.~\ref{alg:ewma}) & $C_\AnyOpNA + m$  & $m + M_\AnyOpNA$ \\
\hline
\multicolumn{1}{c}{\emph{Model-free:}} & & \\
SW (Alg.~\ref{alg:ma}) & $C_\AnyOpNA + m$ & $Bd + m + M_\AnyOpNA$ \\
Scan-$B$~\cite{Li2015b} & $N B d$ & $NBd$ \\
\NEWMA{} (Alg.~\ref{alg:newma}) & $C_\AnyOpNA + m$ & $m + M_\AnyOpNA$ 
\end{tabular}
\end{table}




In this section we introduce the proposed algorithm NEWMA (Alg.~\ref{alg:newma}), give some of its basic theoretical properties, and derive heuristics to choose some of its hyperparameters.

\subsection{The NEWMA algorithm}

NEWMA is based on the following idea: compute \textbf{two EWMA statistics with different forgetting factors} $\updtp<\updt$, and raise an alarm when these two statistics are too far apart. The intuition behind this idea is simple: the statistic with the larger forgetting factor $\updt$ gives ``more importance'' to recent samples than the one that uses $\updtp$, so the distance between them should increase in case of a recent change.

To help illustrate this, in Fig.~\ref{fig:illus_weights}, we schematically represent the three different weighting procedures of EWMA, SW and NEWMA. As mentioned in the introduction: 
1) EWMA computes recursively one average with exponential weights, but requires prior knowledge of a control value to compare with; 
2) SW computes averages in two different time windows, but needs to keep in memory the last $2B$ samples for this purpose; 
and 3) by recursively computing two exponentially weighted averages with different forgetting factors, NEWMA compares pools of recent and old samples (see Prop.~\ref{prop:decomp_sketch} just below), but \emph{does not need to store them in memory}. \rev{By this point of view, NEWMA relies on the same principle as SW (which will be our main baseline for comparison), but is expected to be more efficient computationally.}

\rev{
\begin{remark}\label{rem:KL}
Unlike SW, the NEWMA algorithm is very specific to the use of generalized moments. Indeed, since SW has access to the raw data in two time windows, it could potentially estimate any generic metric between batches of samples, such as the KL-divergence, the total variation~\cite{Kifer2004} or the Wasserstein distance~\cite{Cheng2019}, although we mentioned in the introduction their potential issues in high dimension. On the contrary, NEWMA does not have access to the raw data, and is specifically based on computing on-the-fly generalized moments with different forgetting factors. Although a finite number of generalized moments can never capture all information for all probability distributions, we will see in Section \ref{sec:MMD} how \textbf{\emph{randomly chosen}} moments approximate the MMD, which is a true metric.
\end{remark}
}

In Table~\ref{tab:complexities}, we compare their computational costs, along with the Scan-$B$ algorithm of~\cite{Li2015b} described in Sec.~\ref{sec:related} and used in our experiments.
We can see that the complexity of MA and Scan-$B$ is generally dominated by the storage of the raw data, while \NEWMA{} has the same complexity as \EWMA{}. A crucial factor is the computational cost of $\AnyOpNA$, see Sec.~\ref{sec:MMD} for the case of kernel random features.

\subsection{Preliminary analysis of NEWMA}

Let us formalize a bit the intuition behind NEWMA, by showing first that it indeed computes implicitely a difference between empirical averages in two different time windows. The following, simple proposition is showed in App.~\ref{sec:proof-decomp-sketch} in the supplementary material.
%



\begin{proposition}[\textbf{Rewriting the detection statistic}]
\label{prop:decomp_sketch}
Define $\winsize = \winsize(\updtp,\updt) \eqdef \left\lceil \frac{\log\paren{\updt/\updtp}}{\log\paren{(1-\updtp)/(1-\updt)}}\right\rceil$, and run NEWMA (Alg.~\ref{alg:newma}).
Then, for any $t > \winsize$,
\[
\Sketch_t-\Sketch'_t = C \paren{\sum_{i=t - \winsize+1}^t a_i \AnyOp{\sample_i} - \paren{b_0 \Sketch_0 + \sum_{i=1}^{t-\winsize} b_i \AnyOp{\sample_i}}}
\, ,
\]
where $C = C(\updtp,\updt)\eqdef (1-\updtp)^\winsize - (1-\updt)^\winsize \in (0,~1)$, and $a_i,b_i$ are positive numbers which depend only on~$\updt$ and~$\updtp$, such that $\sum_{i=t-\winsize+1}^t a_i = 1$ and $\sum_{i=0}^{t-\winsize}b_i = 1$.
The exact expressions of
$a_i$ and~$b_i$, 
can be found in App.~\ref{sec:proof-decomp-sketch}. 
\end{proposition}
We see that NEWMA indeed computes the difference between a weighted empirical average of $\AnyOp{\sample_i}$ over the last $\winsize$ samples (where $\winsize$ depends on $\updt$ and $\updtp$) and an empirical average over the samples that came before, and therefore that its behavior intuitively mimics that of SW, without the requirement to store raw data in memory.

Using Prop.~\ref{prop:decomp_sketch} and simple concentration inequalities, we can show basic probabilistic bounds on the detection statistic.
We recall that we designed our algorithm to detect changes through the lens of $\theta_\AnyOpNA(\Prob)=\Exp\AnyOp{\sample}$, and defined a pseudometric $\distphi$ in \eqref{eq:distphi}. The following proposition shows simple ``pointwise'' bounds on $\Sketch_t - \Sketch'_t$ under the null or when there is a change in the last $\winsize$ samples. Its proof, based on Mc Diarmid's concentration inequality, is given in App.~\ref{app:proof_pointwise} in the supplementary material. We note that such pointwise results are different from usual quantity examined in change-point detection such as the mean time between false alarm, which will be examined in Section \ref{sec:ARL}.

\begin{proposition}[\textbf{Bounds at a given time}]
\label{prop:pointwise_detection}
Suppose that $M \eqdef \sup_{\sample\in\RR^d} \norm{\AnyOp{\sample}} < \infty$. 
Let $t>\winsize$ be a fixed time point, and $\rho\in (0,1)$ be some probability of failure. 
\begin{enumerate}[label=(\roman*),leftmargin=15pt]
\item Assume that all samples $\sample_1,\ldots,\sample_t$ are drawn \iid from $\Prob$. 
Then, with probability at least $1-\pFail$, we have 
\begin{equation}\label{eq:pointwise_null}
\norm{\Sketch_t - \Sketch'_t} \leq \epsfund + \epsinit 
\, ,
\end{equation}
where 
$\epsfund = 4\sqrt{2}M\sqrt{(\updt+\updtp)\log(1/\pFail)}$ and $
\epsinit = \left[(1-\updtp)^t - (1-\updt)^t\right] \norm{\Sketch_0 - \Exp_\Prob \AnyOp{\sample}}$.
\item Assume that the last $\winsize$ samples are drawn \iid from a distribution $\Prob'$, and all the samples that came before are drawn \iid from $\Prob$ (that is, $\sample_{t-\winsize},\ldots,\sample_t \simiid \Prob'$ and $\sample_{1},\ldots,\sample_{t-\winsize} \simiid \Prob$). 
Then, with probability at least $1-\pFail$ on the samples, we have 
\begin{equation} \label{eq:pointwise_alter}
\norm{\Sketch_t - \Sketch'_t}\geq C \distphi(\Prob,\Prob') - \epsfund - \epsinit \, .
\end{equation}
where $C$ is defined as in Prop.~\ref{prop:decomp_sketch}.
\end{enumerate}

\end{proposition}

Prop.~\ref{prop:pointwise_detection} shows that, when no change occurs (under the null) the detection statistic is bounded with high probability, and when the last $\winsize$ samples are distributed ``differently'' from the previous ones, it is greater than a certain value with high probability. As expected, this difference is measured in terms of the pseudometric $\distphi$.
Note that a more precise statement can be found in App.~\ref{app:proof_pointwise} in the supplementary material.

\begin{remark}
For the sake of clarity, in Prop.~\ref{prop:pointwise_detection}, $(ii)$, we assumed that \emph{exactly} the last $\winsize$ samples were drawn from $\Prob'$, and that \emph{all} samples that came before were drawn from $\Prob$. In App.~\ref{app:proof_pointwise}, we show a more general result which explicits robustness to slight deviations from this assumption.
\end{remark}

\subsection{Choice of the forgetting factors $\updt$ and $\updtp$}\label{sec:optimal_param}

Although the role of the hyperparameters $(\updt,\updtp)$ in the NEWMA algorithm is simple to understand intuitively, it is not clear how to set their values at this stage. On the contrary, the window size $\winsize$ in Prop. \ref{prop:decomp_sketch} has a more \emph{interpretable} meaning: it is the number of recent samples compared with old ones. \rev{While it is known that the choice of a window size is a difficult problem and that there is no ``universally'' good choice, we assume that practitioners are more familiar with choosing a proper window size (sometimes by simple trial-and-error), than they are with choosing forgetting factors that may be difficult to interpret. 
Hence, in this section, we derive a simple heuristic to set both parameters $(\updt,\updtp)$ for a given $\winsize$, which we assume to be given by the user. Methods to guide the selection of $B$ are left for future investigations.} We build upon the theoretical results of the previous section. We note that choosing a forgetting factor for \EWMA{} is also known to be a difficult problem~\cite{Cisar2011}.

Our starting point is the expression of the window size~$\winsize$ derived in Prop.~\ref{prop:decomp_sketch}. 
We first note that a possible parameterization of NEWMA is through $\winsize$ and \emph{one} of the forgetting factors, say $\updt$: given $\winsize$ and $\updt > \frac{1}{B+1}$, there is a unique $\updtp = \updtp_{\updt,\winsize} \leq \frac{1}{B+1}$ such that $\frac{\log\paren{\updt/\updtp}}{\log\paren{(1-\updtp)/(1-\updt)}} = \winsize$ in Prop.~\ref{prop:decomp_sketch}.
Indeed, $f:x \mapsto x(1-x)^\winsize$ is increasing on $[0,\frac{1}{\winsize+1}]$ and decreasing on $[\frac{1}{\winsize+1},1]$, so the equation $f(x) = f(\updt)$ has exactly one solution in $[0,1]$ besides $\updt$ itself.
Thus $\updtp$ is uniquely defined by $(\winsize,\updt)$. 
We now turn to the choice of $\updt$ given a user-defined window size $\winsize$. 

%
From Prop.~\ref{prop:pointwise_detection}, we can see that the null hypothesis is intuitively distinguishable from the alternative if the bound under the null \eqref{eq:pointwise_null} is smaller than the guaranteed deviation \eqref{eq:pointwise_alter} when there is a change, that is, 
$
\epsfund + \epsinit \leq C \distphi(\Prob,\Prob') - \epsfund - \epsinit\, ,
$
which is equivalent to 
\[
\distphi(\Prob,\Prob') \geq 2(\epsfund + \epsinit)/C\, .
\]
Since we want our algorithm to be sensitive to the smallest possible change in  $\distphi(\Prob,\Prob')$, the previous reasoning suggest that a good choice for $\updt$ is to minimize the right-hand side of this expression. \rev{Note that $\varepsilon_1$ depends on the chosen probability of failure $\rho$: the smaller it is, the larger $\distphi(\Prob,\Prob')$ should be, however at a mild logarithmic rate.}

To obtain our final heuristic, we replace $\norm{\Sketch_0 - \Exp_\Prob \AnyOp{\sample}}$ by the upper bound $2M$ in in the expression of $\epsinit$, and we take $t=2\winsize$; since $\epsinit \xrightarrow[t\to \infty]{} 0$ and intuitively we consider that our algorithm must be ``applicable'' as soon as we have received twice the window size in data.
In definitive, for a user-defined $\winsize$, we propose the following heuristic to choose\footnote{Note that we discard the multiplicative constants as well as $\log\frac{1}{\rho}$, which we found to have negligible effect in practice.} $\updt$:
\[
\updt^\star \!\! = \!\! \argmin_{\updt \in (\frac{1}{\winsize+1},1)} \!\! \frac{\sqrt{\updtp_{\updt,\winsize} + \updt} + (1-\updtp_{\updt,\winsize})^{2\winsize} - (1-\updt)^{2\winsize}}{(1-\updtp_{\updt,\winsize})^\winsize - (1-\updt)^\winsize}
 ,
\]
where we recall that $\updtp_{\updt,\winsize}$ is the unique $\updtp$ such that $\frac{\log\paren{\updt/\updtp}}{\log\paren{(1-\updtp)/(1-\updt)}} = \winsize$. Once $\updt^\star$ is chosen, we naturally set the corresponding $\updtp^\star = \updtp_{\updt^\star,\winsize}$ to respect the window size $\winsize$. 
We note that $\updt^\star$ and $\updtp^\star$ do not have explicit expressions with respect to $B$, but they can easily be approximated by simple one-dimensional optimization schemes. \rev{In practice, given $\updt$ and $B$, we find $\updtp_{\updt,\winsize}$ with a simple gradient descent, and we select $\updt^\star$ with an exhaustive search over a fine grid of $[0,1]$.}

This heuristic is seen to yield good results in practice in many situations. We leave for future work a more rigorous proof of optimality in simpler settings (\emph{e.g.}, Gaussian data).

\section{Choice of $\AnyOpNA$: Random Features}\label{sec:MMD}


Let us now turn to the important choice of the embedding $\AnyOpNA$. 
We recall that $\AnyOpNA$ is user-defined, and that the algorithms studied in this paper are sensitive to variations in the collection of generalized moments $\theta_\AnyOpNA(\Prob) = \Exp_\Prob\AnyOp{\sample}$. As mentioned before, if the practitioner knows in advance which statistic is susceptible to vary, then $\AnyOpNA$ can be chosen accordingly. 
However, one does not necessarily have a priori knowledge on the nature of the change. In this section, we describe a generic embedding related to kernel metric on distributions.

\subsection{Maximum Mean Discrepancy}

For most $\AnyOpNA$, $\distphi(\Prob,\Prob')$ is only a \emph{pseudo}metric on probability distributions: for instance, when $\AnyOp{\sample} = \sample$, it can only distinguish distributions that have different means. 
Ideally however, one would like $\distphi$ to be a true metric, that is, we want $\distphi(\Prob, \Prob') = 0$ if, and only if, $\Prob = \Prob'$. 
Unfortunately, for any mapping $\AnyOpNA$ with values in a finite-dimensional space, $\distphi$ cannot be a true metric---otherwise $\theta_\AnyOpNA(\cdot)$ would be an isometry between an infinite-dimensional space and a finite-dimensional space.
In particular, this is the case for any $\AnyOpNA$ used in practice. 
Luckily, as described in the rest of this section, an interesting strategy is to leverage the Random Features methodology to obtain \emph{random} embeddings $\AnyOpNA$ such that $\distphi(\Prob,\Prob')$ approximates a true metric between distributions \emph{with high probability}.
%

A possible choice for such a metric is the Maximum Mean Discrepancy (MMD,~\cite{Gretton2007}). 
Given a positive definite kernel $\kernelPsi$ on $\RR^d$, take $\MeasSpace$ as the Reproducing Kernel Hilbert Space (RKHS\footnote{A good introduction to the theory of RKHSs is e.g. \cite{Berlinet2004}}) associated to $\kernelPsi$. 
If we set $\AnyOp{\sample} = \kernelPsi(\sample,\cdot)$ and $\norm{\cdot} = \norm{\cdot}_\MeasSpace$, then, in our notation, $\distphi(\Prob, \Prob') = \norm{\Exp_\Prob \kernelPsi(x,\cdot) - \Exp_\Prob \kernelPsi(x,\cdot)}_\MeasSpace$ is the MMD between $\Prob$ and $\Prob'$, that we denote by $\MMD(\Prob,\Prob')$. 
When the kernel $\kernelPsi$ is \emph{characteristic}, it is a true metric. 
Many conditions have been formulated over the years for $\kernelPsi$ to be characteristic~\cite{Sriperumbudur2010}, and for instance the Gaussian kernel is characteristic.
First introduced in the context of two-sample test, the MMD appears quite naturally in the context of kernel change-point detection~\cite{Harchaoui2009, Gar_Arl:2018} and in particular the online Scan-$B$ algorithm~\cite{Li2015b}. 

\subsection{Random Features} \label{sec:RF}
In practice, since $\AnyOp{\sample} = \kernelPsi(\sample,\cdot)$ cannot be stored in memory to compute the theoretical MMD, empirical estimates thereof are used. 
Such estimates usually make use of the so-called \emph{kernel trick}, and require the computation of a $U$-statistic depending on populations drawn from both distributions: it is for instance the method used in the kernel Scan-$B$ algorithm~\cite{Li2015b}.
Since we do not want to store samples when \NEWMA{} is running, least of all perform costly computations on these samples, we resort to kernel Random Features (RF,~\cite{Rahimi2007}), exploiting the simple fact \emph{the Euclidean distance between averaged random features approximates the MMD with high probability over the features}. RFs and MMD have been combined together before, for accelerating the estimation of the MMD~\cite{Sutherland2015} or as a mean to design random projections of distributions in an inverse-problem context~\cite{Gribonval2017}. We also note that alternatives to RFs have been studied in the MMD literature~\cite{Chwialkowski2015}, which is an interesting path for future work.



Let us briefly describe the RF machinery. 
Assume that the kernel $\kernelPsi$ can be written as
$
\kernelPsi(\sample,\sample') = \Exp_{\freq \sim \freqdist} \rfeat(\sample)\overline{\rfeat(\sample')}
$ 
for a family of functions $\rfeat:\RR^d\to \CC$ parameterized by $\freq \in \RR^q$, and a probability distribution $\freqdist$ on $\RR^q$. 
This is for instance the case for all translation-invariant kernels~\cite{Ber_Chr_Res:1984, Rahimi2007}: \rev{by Bochner's theorem, they can all be written under this form for complex exponentials $\rfeat(\sample) = e^{\textrm{i}\freq^\top \sample}$ and some symmetric distribution $\Gamma$. Using complex exponentials as $\rfeat$ is usually referred to as Random \emph{Fourier} Features (RFF). The most classical example is the Gaussian kernel $\kernelPsi(\sample, \sample') = e^{-\frac{\norm{\sample - \sample'}^2}{2\sigma^2}}$, which is written under this form for a Gaussian distribution $\Gamma = \mathcal{N}(0, \sigma^{-2} \Id)$.

For some (large) integer $m \in \mathbb{N}$, the RF paradigm consists in drawing $m$ parameters $\freq_1,\ldots,\freq_m\simiid \freqdist$ and defining $\AnyOpNA:\RR^d \to \CC^m$ as}
\begin{equation}\label{eq:defRF}
\AnyOp{\sample} \eqdef \frac{1}{\sqrt{m}}\paren{\feat_{\freq_j}(\sample)}_{j=1}^m
\, ,
\end{equation}
and taking $\norm{\cdot}$ as the classical Hermitian norm on $\CC^m$. A simple computation (see the proof of Prop.~\ref{prop:MMD-RF} in App.~\ref{app:proof_MMD}) then shows that $\distphi(\Prob,\Prob')\approx \MMD(\Prob,\Prob')$, with high probability over the $\freq_j$.
With this choice of $\AnyOpNA$, we have the following result similar to Prop.~\ref{prop:pointwise_detection}.

\begin{proposition}[\textbf{EWMA-RF pointwise detection}]
\label{prop:MMD-RF}
Suppose that $\sup_{\sample,\freq} \abs{\rfeat(\sample)} \leq M$. 
Define $\AnyOp{\cdot}$ as in Eq.~\eqref{eq:defRF}. 
Let $\pFail\in (0,1)$ be a probability of failure. 
Suppose that the assumptions of Prop.~\ref{prop:pointwise_detection}, $(ii)$ hold. Then, with probability at least $1-2\pFail$ on both samples $\sample_i$ \emph{and} parameters $\freq_j$, it holds that
\begin{align}
\label{eq:pointwise_MMD}
\norm{\Sketch_t - \Sketch'_t}\geq C \Big(\MMD^2(\Prob,\Prob') -\varepsilon_m \Big)_+^\frac12 - \epsfund - \epsinit
\, ,
\end{align}
where $(x)_+ = \max(x,0)$ and $\varepsilon_m = \tfrac{2\sqrt{2}M^2}{\sqrt{m}}\sqrt{\log\tfrac{1}{\pFail}}$.
%
\end{proposition}
By the previous proposition, if the MMD between $\Prob$ and $\Prob'$ is large, then with high probability so is the deviation of $\norm{\Sketch_t - \Sketch'_t}$. The additional error $\sqrt{\varepsilon_m}$ is of the order of the previous error $\epsfund$ if $m = \order{(\updt + \updtp)^{-2}}$. 


The choice of a good kernel $\kernelPsi$ is a notoriously difficult problem. Ideally, one would choose it so as to maximize $\MMD(\Prob,\Prob')$, however neither $\Prob$ nor $\Prob'$ are known in advance in our setting. 
In practice, we use the Gaussian kernel.
Having access to some initial training data, we choose the bandwidth~$\sigma$  using the median trick as in~\cite{Li2015b}. 
We leave for future work more involved methods for kernel selection~\cite{Yang2012a}.

\subsection{Fast random features and computational cost}\label{sec:RFcost}

A crucial factor in the application of \NEWMA{} is the complexity of the mapping $\AnyOpNA$, both in computation time or memory footprint of the parameters necessary to compute it, which we respectively denoted by $C_\AnyOpNA$ and $M_\AnyOpNA$ in Table \ref{tab:complexities}. For usual RFFs~\cite{Rahimi2007}, computing $\AnyOp{\sample}$ requires storing the dense matrix of frequencies $\freq_j \in \RR^d$, and performing a costly matrix-vector product. Therefore, in this case both $C_\AnyOpNA$ and $M_\AnyOpNA$ scale as $\order{md}$, which somehow mitigates the computations advantages of using NEWMA over more costly methods.

However, a large body of work is dedicated to accelerate the computation of such random features. 
For instance, the Fastfood (FF) approximation~\cite{Le2013} reduces the time complexity to $\order{m\log d}$ and memory to $\order{m}$. 

More strikingly, in~\cite{Saade2016}, the authors build an Optical Processing Unit (OPU), to which we had had access for our experiments (Sec.~\ref{sec:expe}), that computes random features in $\order{1}$ and eliminates the need to store the random matrix. Let us briefly describe how the OPU operates. To compute some random mapping $\paren{\feat_{\freq_j}(\sample)}_{j=1}^m$, the signal $\sample$ is first encoded into a \emph{light beam} using a Digital Micromirror Device (DMD). The beam is then focused and \emph{scattered} through a layer of heterogeneous material, which corresponds to performing many random linear operations on the underlying signal. Then, the amplitude is measured by a camera, which adds a non-linearity on the output. The corresponding kernel, which is imposed by the physical device, is an elliptic kernel whose expression is given in~\cite{Saade2016}. We refer the reader to~\cite{Saade2016, Dremeau2015} for complete details on the process. In addition to being able to compute RFs in $\order{1}$ for $m,d$ in the order of millions (for current hardware), OPUs are also significantly less energy-consuming than classical GPUs.

We summarize the respective complexities of these three approaches in Table~\ref{tab:RF}.


\begin{table}[h]
\caption{\label{tab:RF}Time complexity $C_\AnyOpNA$ and memory requirement $M_\AnyOpNA$ for different Random Features schemes.}
\centering
\small
\begin{tabular}{l|lll}
& \textbf{RFF} & \textbf{FF}  & \textbf{OPU} \\ 
\hline 
$C_\AnyOpNA$ & $\order{md}$  & $\order{m \log d}$ & $\order{1}$ \\ 
$M_\AnyOpNA$ & $\order{md}$ & $\order{m}$ & $\order{1}$ \\ 
\end{tabular}
\end{table}

\section{Setting the threshold}\label{sec:thres}

In this section, we go back to the case of any general mapping $\AnyOpNA$.
An important question for any change-point detection method that tracks a statistic along time is how to set the threshold $\thres$ above which a change-point is detected. 
In this section, we begin by adapting to NEWMA two classical approaches that use the property of the algorithm under the null hypothesis. 
However, while they are interesting in their own right, these approaches generally cannot be directly used in practice since they require to know the in-control distribution $\Prob$. Hence we describe an efficient numerical procedure to dynamically adapt the threshold during a continuous run with multiple changes. 

\subsection{Mean time between false alarms}\label{sec:ARL}

A classical method to set the threshold $\thres$ is to adjust a desired mean time between false alarms under the null, defined as
%
%
$
\Tfalsealarm \eqdef \Exp\brac{\inf\set{t~|~t\text{ is flagged}}}, 
$
%
where the expectation is over the samples under the null (that is, drawn \iid from some distribution $\Prob$).
In the literature, it is often referred to as the Average Run Length (ARL) under control.

\edit{Unless strong assumptions are made on $\Prob$, it is often impossible to derive a closed-form expression for the ARL. A possible strategy is to estimate it using some training data, however this method is impractical in a continuous run with multiple changes. We will derive simpler strategies in the next sections.
For theoretical purposes, we nevertheless show that it is possible to adapt the Markov chain-based proof developed for classical EWMA in~\cite{Fu2002} to NEWMA: this method derives an expression for the ARL -- however, as we mentioned, it rarely has a closed-form expression and requires unreasonable prior knowledge. Unlike the results from the previous sections, our analysis is valid without any boundedness assumption on $\AnyOpNA$. We present our theorem in the unidimensional case $\MeasSpace = \RR$.

\begin{theorem}[\bf Average Run Length of NEWMA]\label{thm:ARLuni}
Assume that $\AnyOpNA:\RR^d \to \RR$ maps to a unidimensional space, and assume that $\AnyOp{X}$ has a density under the null. Denote by $F(x) \eqdef \PP_{X\sim \Prob}\paren{\AnyOp{X}\leq x}$ its cumulative distribution function. For any $\varepsilon>0$, define $\{a_1,\ldots a_M\}$ an $\varepsilon$-grid of $[-1/\varepsilon, 1/\varepsilon]$, that is, $M \eqdef \left\lceil \frac{2}{\varepsilon^2} \right\rceil$ and $a_i \eqdef (i-1)\varepsilon - 1/\varepsilon$. Then, consider the list of all couples $\vu_k = (a_i, a_j)$ such that $\abs{a_i - a_j} \leq \thres$, indexed by $1\leq k\leq K$ for some $K\leq M^2$ that depends on $\varepsilon$ and $\thres$. For any $\vu_{k_1} = (a_{i_1}, a_{j_1})$ and $\vu_{k_2} = (a_{i_2},a_{j_2})$, define
\[
p_{k_1 k_2} = \begin{cases} F(u_2) - F(u_1) &\quad \text{if $u_1 < u_2$} \\
0 &\quad\text{otherwise.} \end{cases}
\]
where
\[
\begin{cases}
u_1 &= \max\Big(\frac{1}{\updt}(a_{i_2}-(1-\updt)a_{i_1}-\varepsilon/2),\\
&\qquad\qquad\frac{1}{\updtp}(a_{j_2}-(1-\updtp)a_{j_1}-\varepsilon/2)\Big) \, , \\
u_2 &= \min\Big(\frac{1}{\updt}(a_{i_2}-(1-\updt)a_{i_1}+\varepsilon/2),\\
&\qquad\qquad\frac{1}{\updtp}(a_{j_2}-(1-\updtp)a_{j_1}+\varepsilon/2)\Big) \, .
\end{cases} 
\]
Define the matrix $\mP = [p_{k_1k_2}]_{k_1,k_2= 1}^K$. Then, we have
\begin{equation}\label{eq:ARLuni}
\Tfalsealarm = 1+ \sum_{\ell\geq 1} \lim_{\varepsilon\to 0}\ve_1^\top \mP^\ell \mathbf{1}_K
\, ,
\end{equation}
where $\Tfalsealarm$ is the ARL of NEWMA, $\mathbf{1} = [1,\ldots,1]^\top$ and $\ve_1 = [1,0,\ldots,0]^\top$.
\end{theorem}
In App.~\ref{app:ARL} we prove a \rev{(quite notation-heavy)} more general version of this theorem in the case where $\AnyOpNA:\RR^d \to \RR^m$ is a multidimensional map, in which case the grid $\{a_i\}$ is replaced by an $\varepsilon$-net.

\rev{From \eqref{eq:ARLuni}, it is difficult to describe precisely the effect of the different parameters of NEWMA on its ARL. Naturally, the larger the threshold $\tau$ is, the higher the ARL is (in Theorem \ref{thm:ARLuni}, a higher threshold results in a larger $K$). Similarly, a larger window size $B$ in Proposition \ref{prop:decomp_sketch} intuitively results in a ``smoother'' algorithm and a higher ARL, although it is less obvious in the theoretical expression. In simple cases where $F$ is known, it is possible to perform numerical simulations using \eqref{eq:ARLuni}. In general,} it is impossible to exchange the infinite sum ``$\sum_{\ell\geq 1}$'' and the limit ``$\varepsilon\to 0$'' in \eqref{eq:ARLuni}, as this would require uniform convergence. However, in practice, for numerical purpose, we can fix a small $\varepsilon>0$ and use the identity $\sum_{\ell\geq 0} \mA^\ell = (\Id-\mA)^{-1}$ to approximate $\Tfalsealarm \approx \ve_1^\top (\Id - \mP)^{-1}\mathbf{1}$.
We illustrate this principle in Fig.~\ref{fig:ARL} with a Gaussian distribution $\AnyOp{X} \sim \mathcal{N}(0,1)$, using respectively Th.~\ref{thm:ARLuni} and the original approach by~\cite{Fu2002}, as well as numerical simulations of runs on synthetic data.}
As we mentioned before, when the in-control distribution $\Prob$ is not known (or, more precisely, when the cumulative distribution function $F$ is not known), Th.~\ref{thm:ARLuni} cannot be directly applied. In some cases~\cite{Li2015b}, one can obtain an \emph{asymptotic} expression for $\Tfalsealarm$ when $\tau \to \infty$ which does not depend on $\Prob$, which we leave for future work. 

\begin{figure}[ht]
\centering
\includegraphics[width=0.43\textwidth]{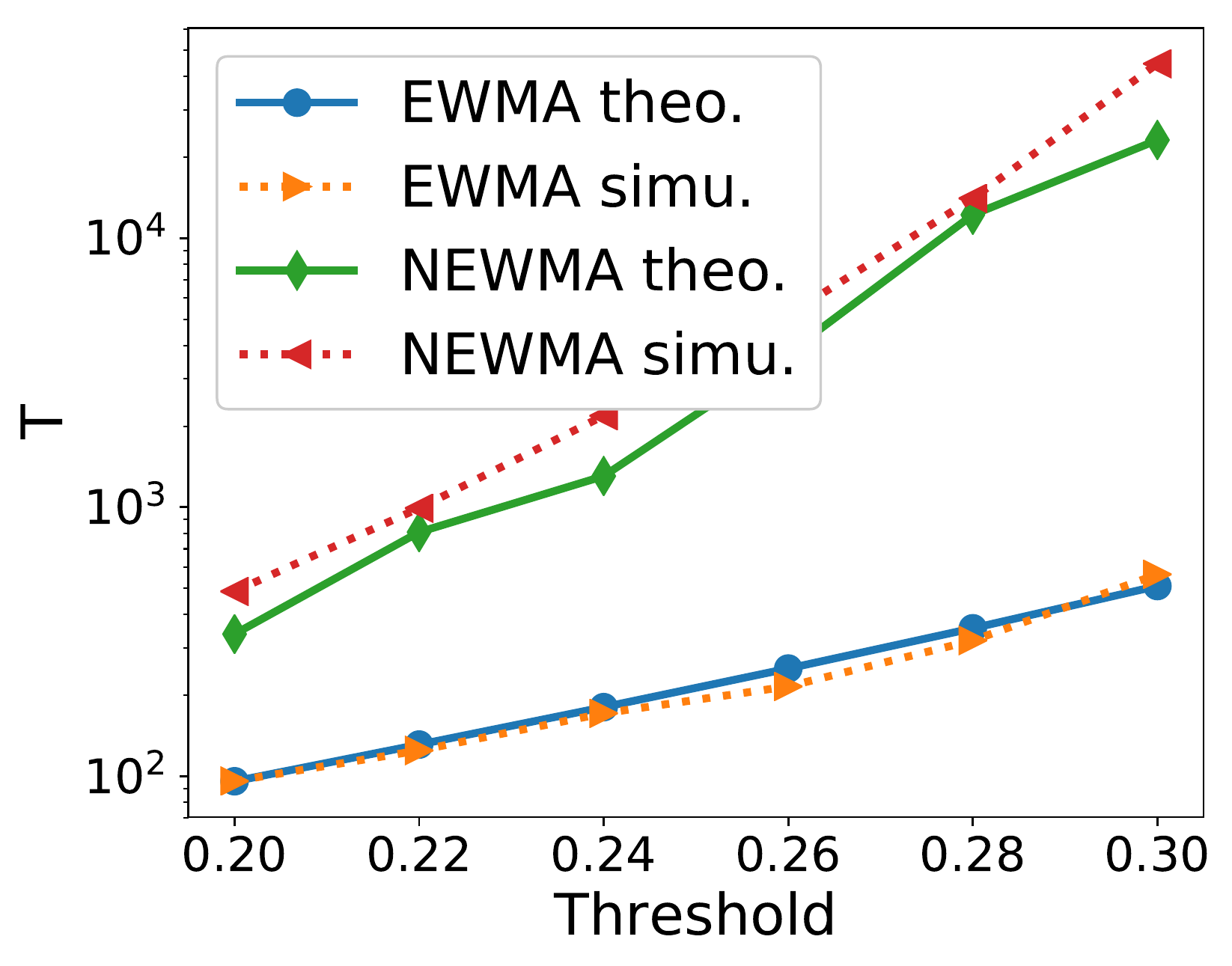}
	\caption{Comparison of the theoretical and observed values of $\Tfalsealarm$ for \NEWMA{} (resp. \EWMA{}) with respect to $\thres$,  
with $\Prob = \mathcal{N}(0,1)$, $\AnyOp{\sample} = \sample$, $\updt = 2\cdot 10^{-1}$ and $\updtp=10^{-1}$. 
The simulations are averaged over $1000$ runs, the theoretical expression is obtained with Th.~\ref{thm:ARLuni} (resp.~\cite{Fu2002}) with a grid of precision $\varepsilon=2\cdot 10^{-2}$. \rev{Here we do not qualitatively compare EWMA and NEWMA, but rather illustrate the quality of the theoretical approximations for the ARL.}}
\label{fig:ARL}
\end{figure}

\subsection{Asymptotic distribution under the null}

\edit{Another, arguably simpler} approach to set the threshold $\thres$ is to derive the distribution of $\norm{\Sketch_t - \Sketch_t'}$ under the null, and set $\thres$ to obtain a desired probability of exceeding it.
In this section, we derive an asymptotic result on the distribution of this statistic when $\updtp\to 0$ and $t \to \infty$. Unlike Prop.~\ref{prop:pointwise_detection}, where $\AnyOpNA$ is assumed uniformly bounded, it relies on the slightly weaker assumption that $\AnyOpNA$ has a finite fourth order moment.

\begin{theorem}[\textbf{Convergence under the null}]
\label{thm:limit}
Assume $\updtCst = \updt/\updtp >1$ is fixed, and let $\updt \to 0$, with $t\geq \frac{2}{\updtp}\log\frac{1}{\updtp}$.
Assume that all samples $\sample_i$ are drawn \iid from $\Prob$. 
Suppose that $\Exp_{\Prob} \norm{\AnyOp{\sample}}^4 < +\infty$.
Set $\mu \eqdef \Exp_\Prob \AnyOp{\sample}$, and $\kernel(\sample,\sample') \eqdef \inner{\AnyOp{\sample}-\mu,\AnyOp{\sample'}-\mu}_\MeasSpace$. 

Define the eigenvalues and eigenvectors of $\kernel$ in $L^2(\Prob)$, \ie, define $\eigen_\ell\geq 0$ and $\eigenfunc_\ell \in L^2(\Prob)$ such that $\kernel(\sample,\sample') = \sum_{\ell\geq 1}\eigen_\ell \eigenfunc_\ell(\sample)\eigenfunc_\ell(\sample')$ and $\inner{\eigenfunc_\ell,\eigenfunc_{\ell'}}_{L^2(\Prob)} = 1_{\ell = \ell'}$.

Then, 

\begin{equation}
\frac{1}{\updtp}\norm{\Sketch'_t-\Sketch_t}^2 \xrightarrow[\eta \to 0]{\mathcal{L}} Y \eqdef  \frac{(1-\updtCst)^2}{2(1+\updtCst)}\sum_{\ell\geq 1}\eigen_\ell W_\ell^2
\, ,\label{eq:conv}
\end{equation}
where $\left(W_\ell\right)_{\ell \geq 1}$ is an infinite sequence of independent standard normal random variables.
\end{theorem}

\begin{figure}[ht]
\centering
\includegraphics[width=0.4\textwidth]{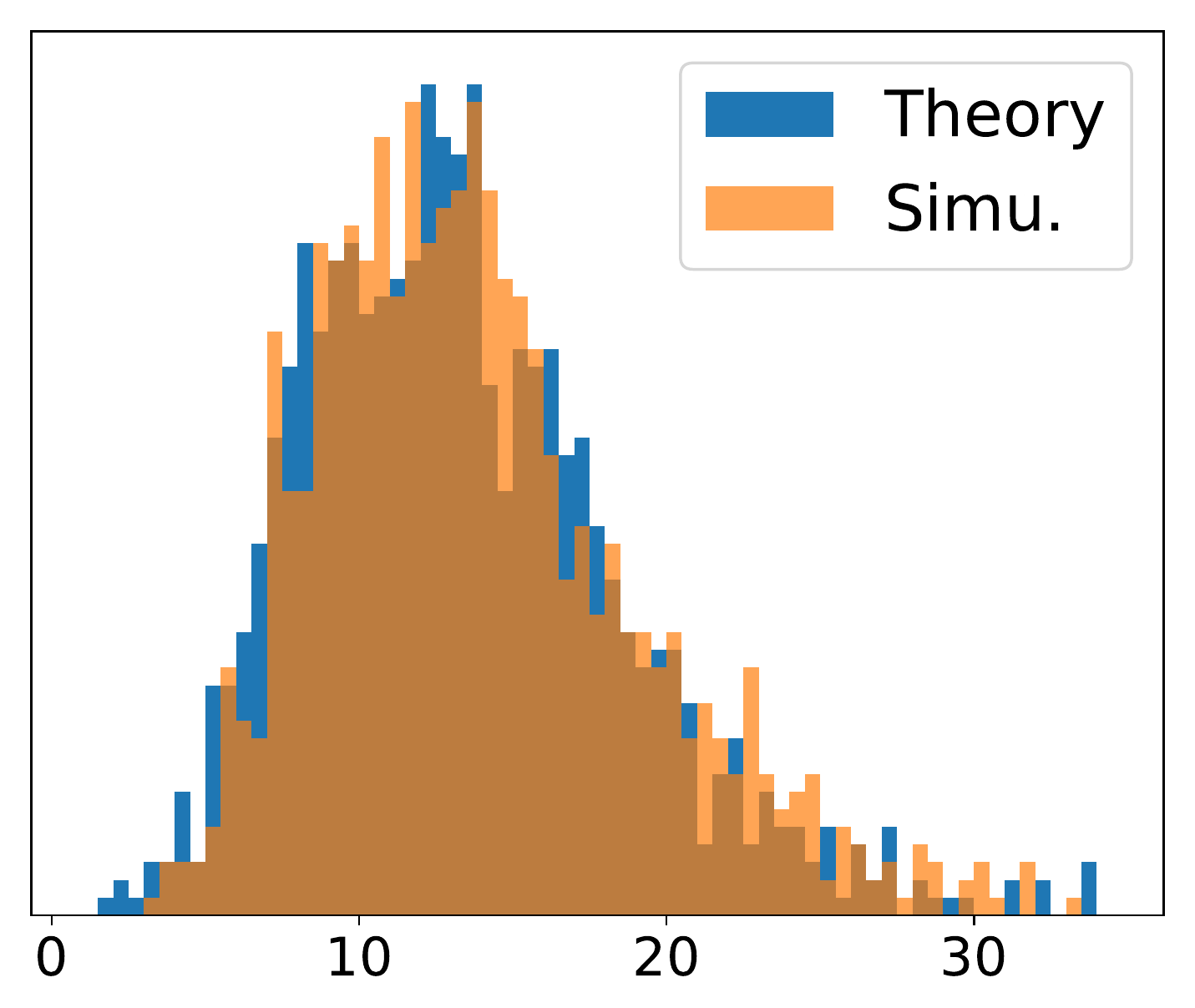}
\caption{Distribution of $\frac{1}{\updtp}\norm{\Sketch_t-\Sketch'_t}^2$ when $\updtp\to 0$ as predicted by Th.~\ref{thm:limit}, on a toy example. Namely,~$\Prob$ is the uniform distribution on $[0,1]$ and $\AnyOpNA(\sample) = \brac{\sqrt{\eigen_\ell}\eigenfunc_\ell(\cdot)}_{\ell=1}^{30}$ is defined as a collection of eigenfunctions $\eigenfunc_{\ell}(x) = \sqrt{2}\cos(2\pi \ell x)$, where the eigenvalues $\left(\eigen_{\ell}\right)_{\ell=1}^{30}$ are randomly generated. We perform $1000$ simulations of both Eq.~\eqref{eq:conv} and \NEWMA{} with $\updt=2\cdot 10^{-2}, \updtp = 10^{-2}$.}
\label{fig:limitdistrib}
\end{figure}

The proof, given in App. \ref{app:limit}, follows closely~\cite{Serfling1980} (Sec.~5.5.2) adapted to our setting, with the use of a multivariate version of Lindeberg's central limit theorem (recalled as Th.~\ref{thm:lindeberg} in the supplementary) instead of the classical Central Limit Theorem.
Th.~\ref{thm:limit} allows to set the threshold $\thres$ if the eigenvalues $\eigen_\ell$ are (approximately) known, for instance they can be estimated using the Gram matrix of $\kernel$ on training data~\cite{Gretton2009}, which we leave for future work. In Fig.~\ref{fig:limitdistrib}, we illustrate the result on a toy example. 

\subsection{Choice of an adaptive threshold}\label{sec:adapt_thres}

\begin{figure}
\begin{algorithmic}
\small
 \REQUIRE Stream of detection statistics $S_t$, estimation rate $0<\alpha<1$, coefficient $a$
\STATE Initialize $\mu_0 = 0$ (mean), $\mu^{(2)}_0 = 0$ (second order moment)
 \FOR{$t=1,2,\ldots$}
  \STATE $\mu_t = (1-\alpha) \mu_{t-1} + \alpha S_t^2$
  \STATE $\mu^{(2)}_t = (1-\alpha) \mu^{(2)}_{t-1} + \alpha S_t^4$
  \STATE $\sigma_t = \sqrt{\mu^{(2)}_t - \mu_t^2}$
  \IF{$S_t^2 \geq \mu_t + a \sigma_t $}
  \STATE Flag $t$
  \ENDIF
  \ENDFOR
 \end{algorithmic}
 \caption{Adaptive Threshold procedure for any online change-point algorithm $S_t$, under the assumption that $S_t^2$ is approximately Gaussian.}
 \label{alg:adapt_thres}
\end{figure}

The two strategies presented above are generally difficult to carry out in practice, since we do not know the in-control distribution $\Prob$. \edit{If we had access to training data, a classical method would be to estimate either the ARL or the distribution of $S_t \eqdef \norm{\Sketch'_t-\Sketch_t}$ during a ``Phase 1'' estimation, before the actual run. In the case where we do not have access to training data beforehand, we propose to adapt the second strategy and perform an \emph{online} estimation of the distribution of the statistic $S_t$, which yields a dynamic threshold $\thres_t$ that can adapt to multiple changes in a continuous run (Fig.~\ref{fig:adapt}).

According to Thm.~\ref{thm:limit}, $S_t^2$ asymptotically follows a distribution formed by a linear combination of an infinite number of independent centered normalized Gaussians with unknown weights $\xi_\ell$. While it would be possible to estimate these weights with relatively heavy computations by maintaining a Gram matrix~\cite{Gretton2009}, in the spirit of the paper we propose a light method that assume that $S_t^2$ itself is approximately Gaussian: indeed, it is easy to see that with additional assumptions on the $\xi_\ell$, generalizations of the Central Limit Theorem (see e.g. Th. \ref{thm:lindeberg}) would guarantee that with proper normalization the r.h.s. of \eqref{eq:conv} converges to a Gaussian (details are omitted here).

Hence, if we consider $S_t^2$ to be Gaussian, we just need to estimate its mean $\mu_t$ and standard deviation $\sigma_t$, which we do with a simple online estimation procedure using exponential weights with a learning rate $\alpha$ (Alg.~\ref{alg:adapt_thres}), to continuously adapt in the case of multiple changes. Then the threshold at time $t$ is set as $\thres_t^2 = \mu_t + a \sigma_t$, where $a$ is chosen according to the desired quantile of the normal distribution (\eg, $a=1.64$ for $5\%$ of false alarms). \rev{Recall that the amplitude of a change between $\Prob$ and $\Prob'$ is approximately described by \eqref{eq:pointwise_alter} (and \eqref{eq:pointwise_MMD} when using RFs), and that we can expect a successful detection when it is higher than the threshold.}

While this method relies on a heuristic than may not necessarily be satisfied, we found in practice this adaptive procedure to perform better than any fixed threshold, while avoiding having to set it. In our experiments, we applied the same strategy to other change-point detection algorithms that produce positive statistics $S_t$ such as SW or Scan-$B$, and found the procedure to perform extremely well in each case. We leave its theoretical analysis for future work, and emphasize again that, in case where training data is available, more complex or computationally intensive procedures could be used.
}

\begin{figure}[ht]
\centering
\includegraphics[width = 0.43\textwidth]{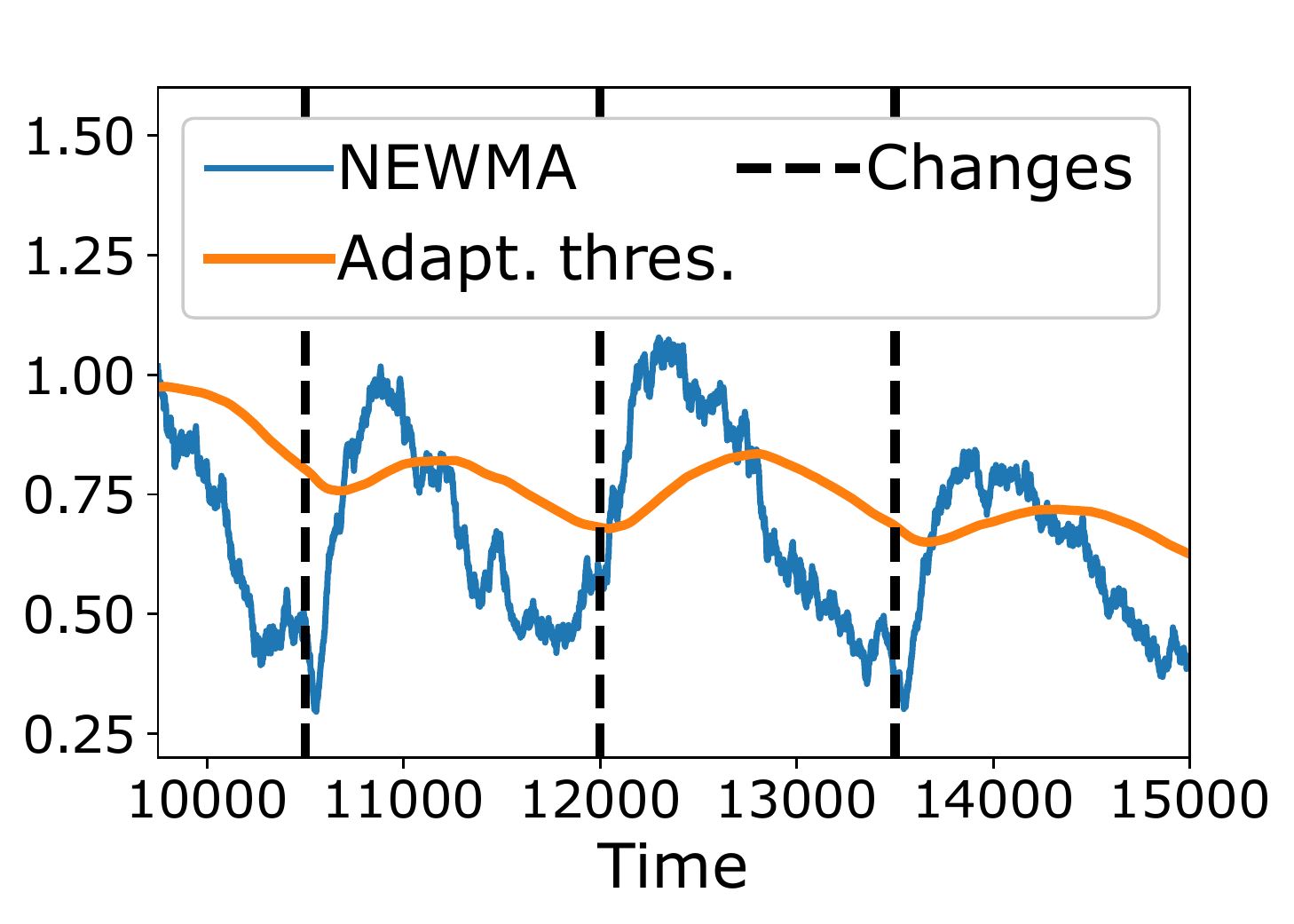}
\caption{Illustration of the adaptive threshold procedure. The dotted line indicate a change, the blue line is the NEWMA statistic $\norm{\Sketch_t - \Sketch'_t}$, and in yellow line is the adaptive threshold computed online as described in Sec.~\ref{sec:adapt_thres}.}
\label{fig:adapt}
\end{figure}

\section{Experiments}\label{sec:expe}

In this experimental section we compare several model-free approaches: \NEWMA{} where $\AnyOpNA$ is one of the three different random features schemes described in Sec.~\ref{sec:MMD}: classical RFFs~\cite{Rahimi2007}, FF~\cite{Le2013}, or OPU~\cite{Saade2016}, the SW algorithm (Alg. \ref{alg:ma},~\cite{Kifer2004}) with RFFs, and the kernel Scan-$B$ algorithm~\cite{Li2015b} with $N=3$ windows.
Scan-$B$ is implemented with a Gaussian kernel \rev{$\kernelPsi(\sample, \sample') = e^{-\frac{\norm{\sample-\sample'}^2}{2\sigma^2}}$} with a bandwidth $\sigma$ chosen by the median trick. All other methods use RFs that correspond to the same kernel, \rev{either complex exponentials for RFFs (see Section \ref{sec:RFcost}) or the Fastfood (FF) method~\cite{Le2013}}, except when using the OPU, for which the RFs and corresponding kernel are imposed by the optical hardware~\cite{Saade2016}.

\begin{remark}
\rev{We do not compare NEWMA with parametric methods such as GLR or CUSUM, or methods requiring prior knowledge such as EWMA, since the settings are very different and fair comparison would be difficult. In the presence of parametric modelling assumptions or prior knowledge, we naturally expect the algorithms exploiting them to perform better than model-free methods such as SW, Scan-$B$ or NEWMA.}
\end{remark}

The experiments run on a laptop with an Intel Xeon Gold 6128 3.40GHz. 
%
The code is available at \url{https://github.com/lightonai/newma}.

\begin{figure}[h!]
\centering
\includegraphics[width = 0.4\textwidth]{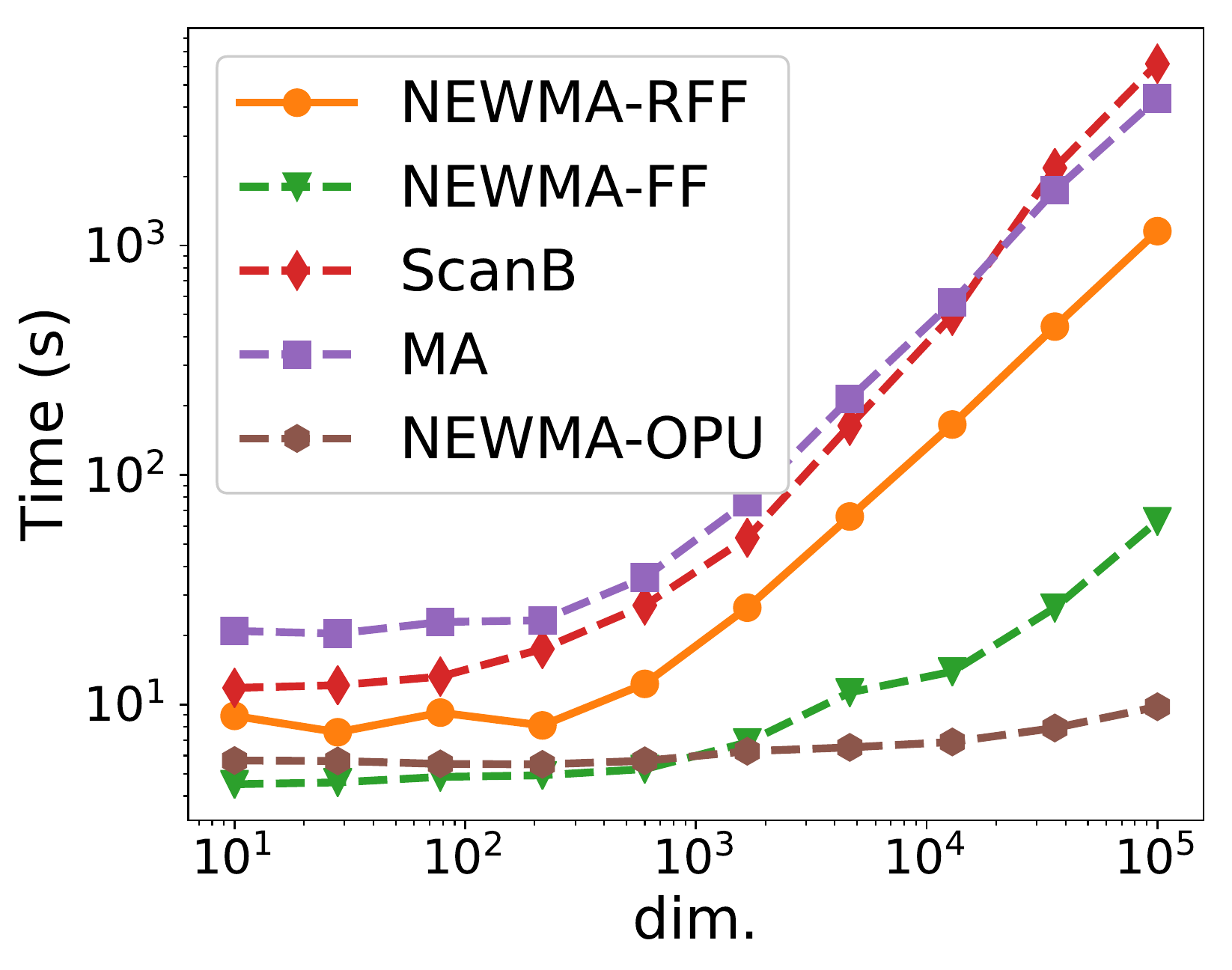}
\includegraphics[width = 0.4\textwidth]{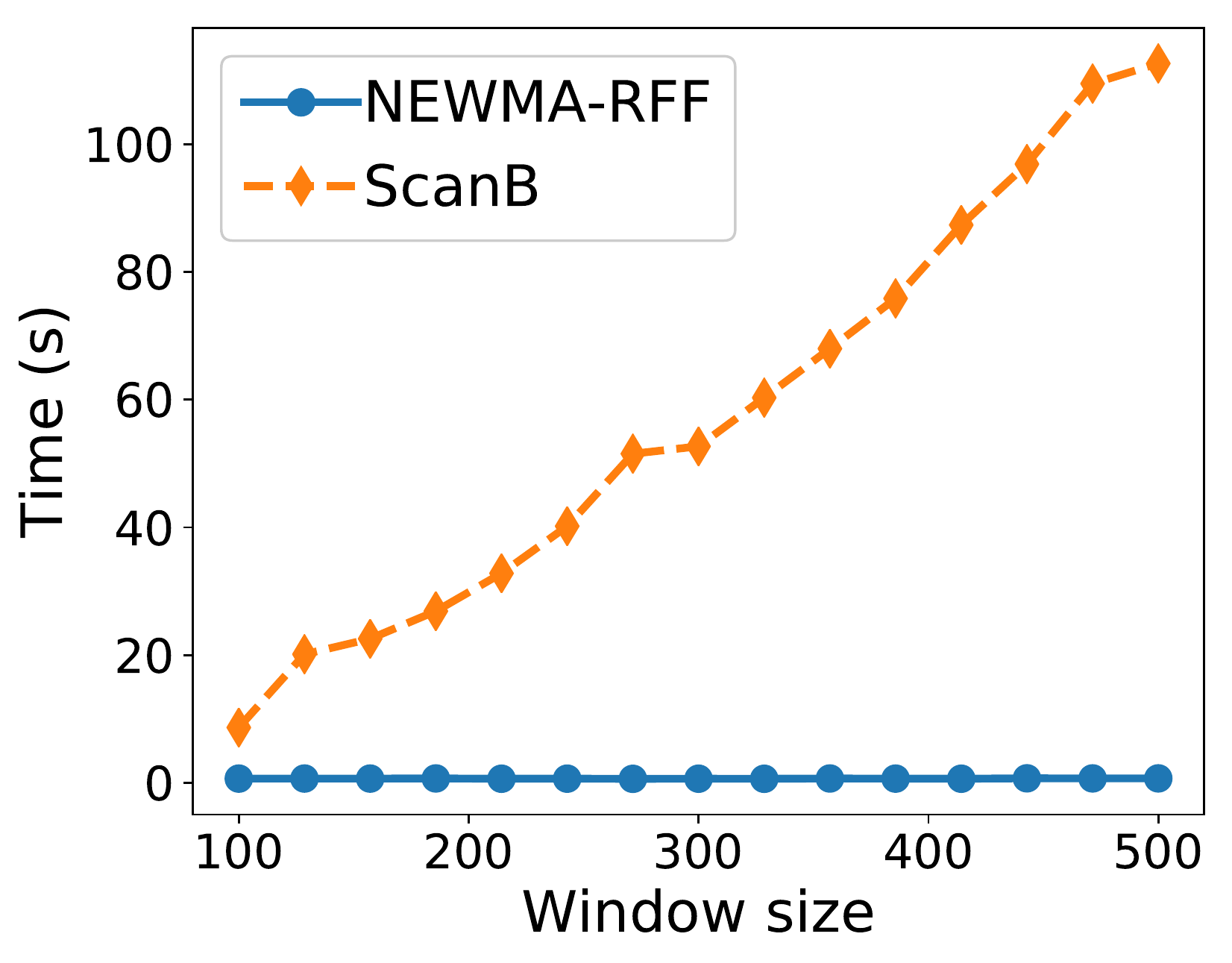}
\caption{Time of execution on $12000$ samples with $m=3000$ RFs for NEWMA and SW, window size $\winsize=250$ for Scan-$B$ and SW, and dimension $d=100$ unless otherwise precised.}
\label{fig:time}
\vspace{-10pt}
\end{figure}

\begin{figure}[h!]
\centering
\begin{subfigure}{0.3\textwidth}
\includegraphics[width = \textwidth]{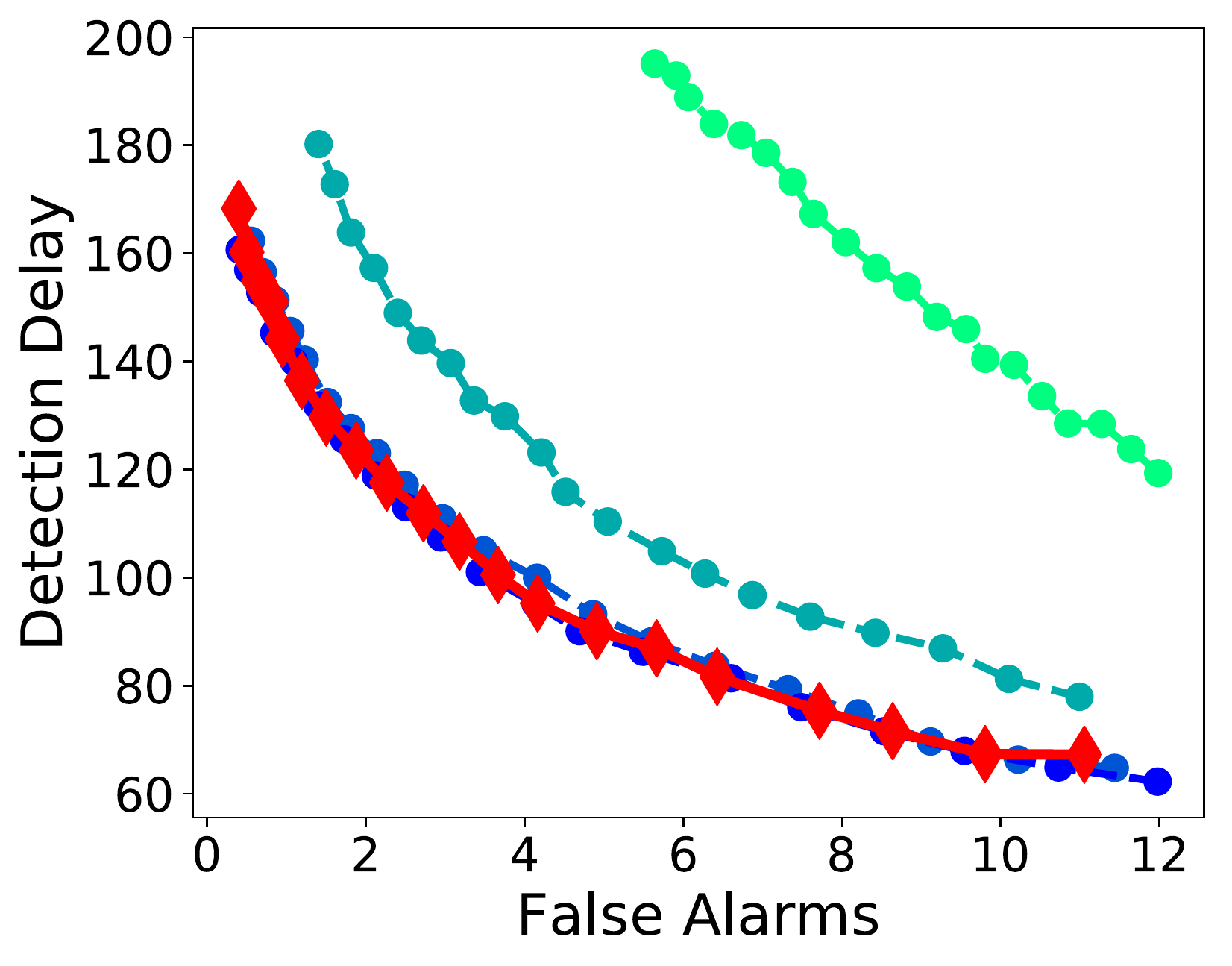} \\
\includegraphics[width = \textwidth]{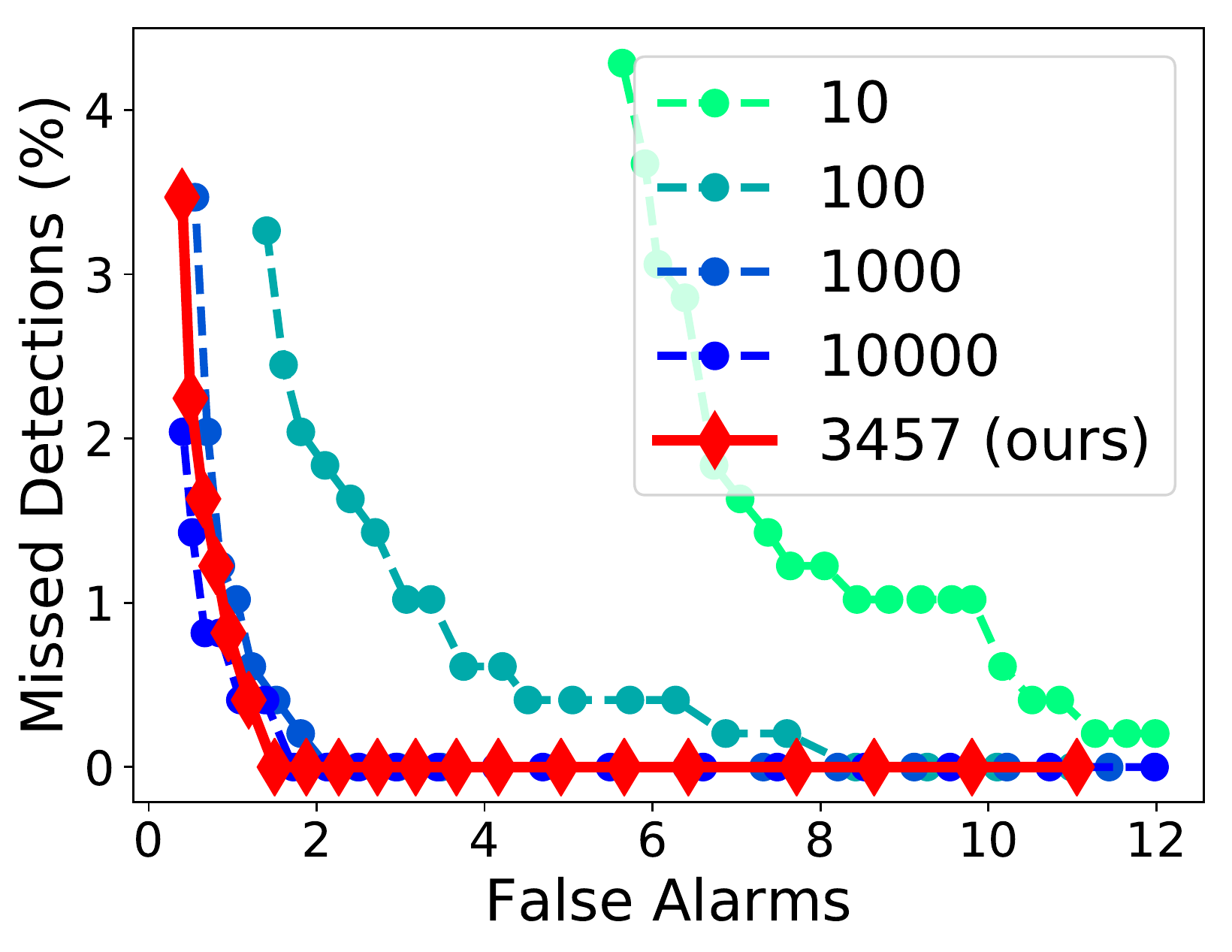}
\caption{Number of random features $m$.}
\label{subfig:m}
\end{subfigure}~
\begin{subfigure}{0.3\textwidth}
\includegraphics[width = \textwidth]{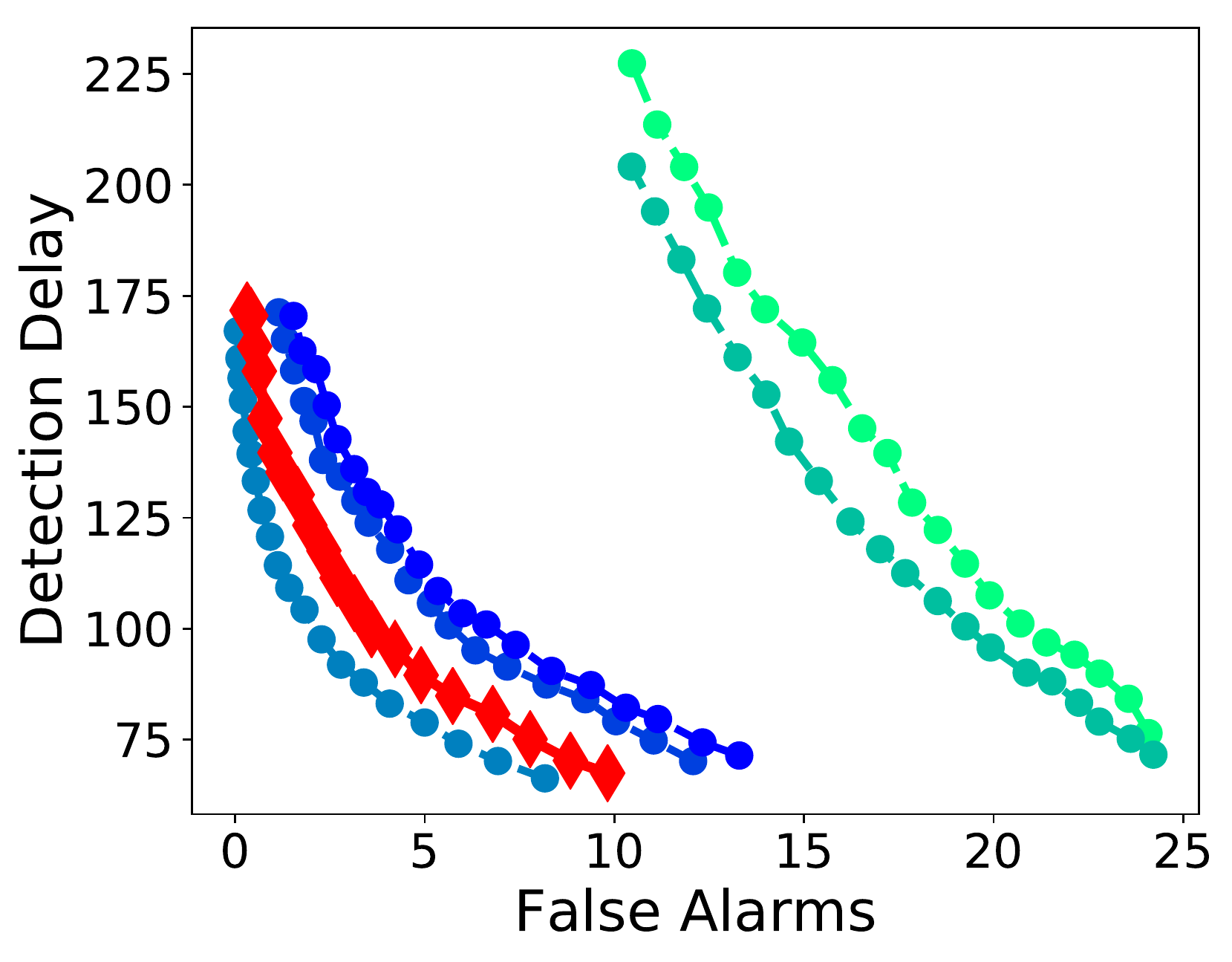} \\
\includegraphics[width = \textwidth]{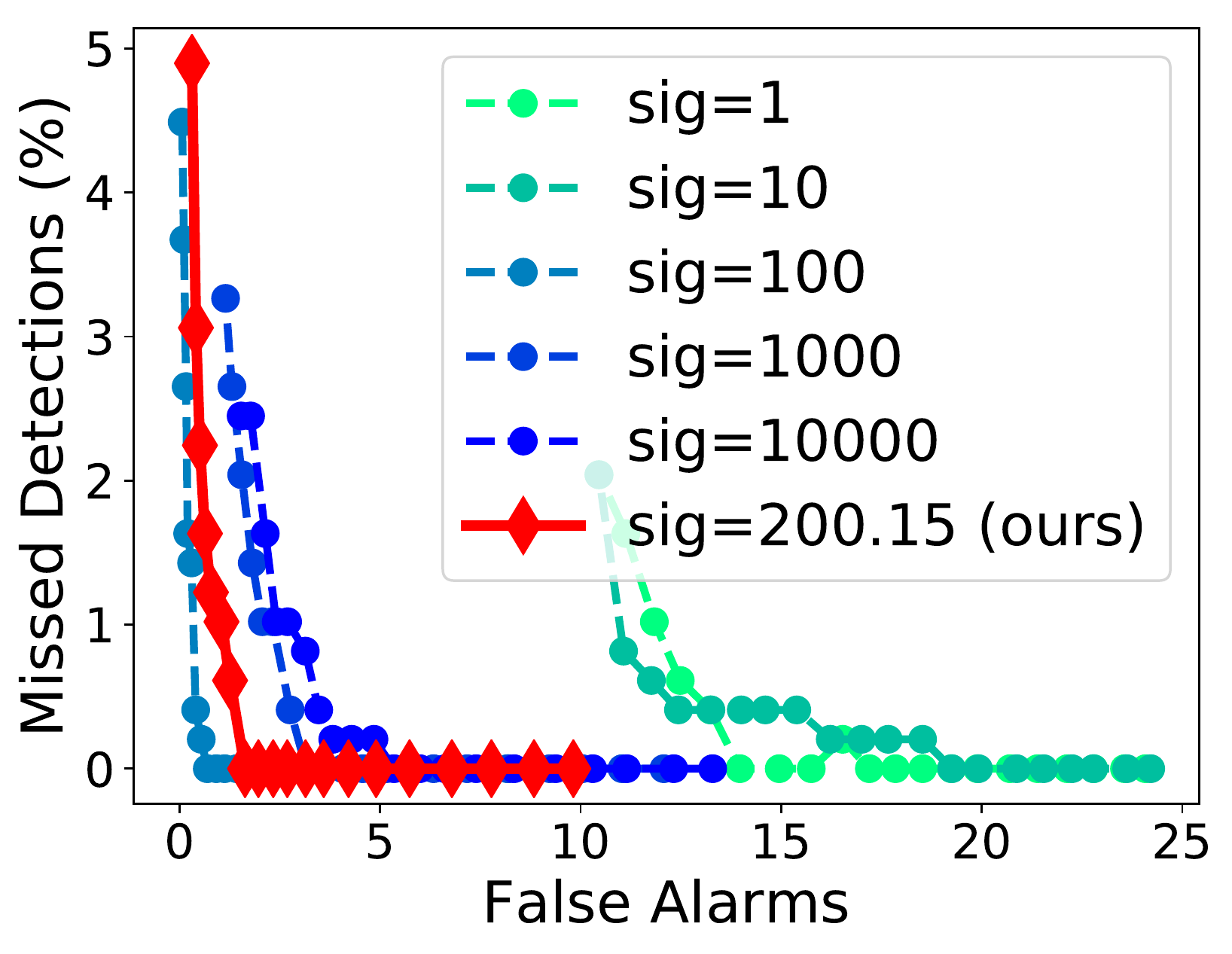}
\caption{Choice of $\sigma$ for Gaussian kernel.}
\label{subfig:sigma}
\end{subfigure}~
\begin{subfigure}{0.3\textwidth}
\includegraphics[width = \textwidth]{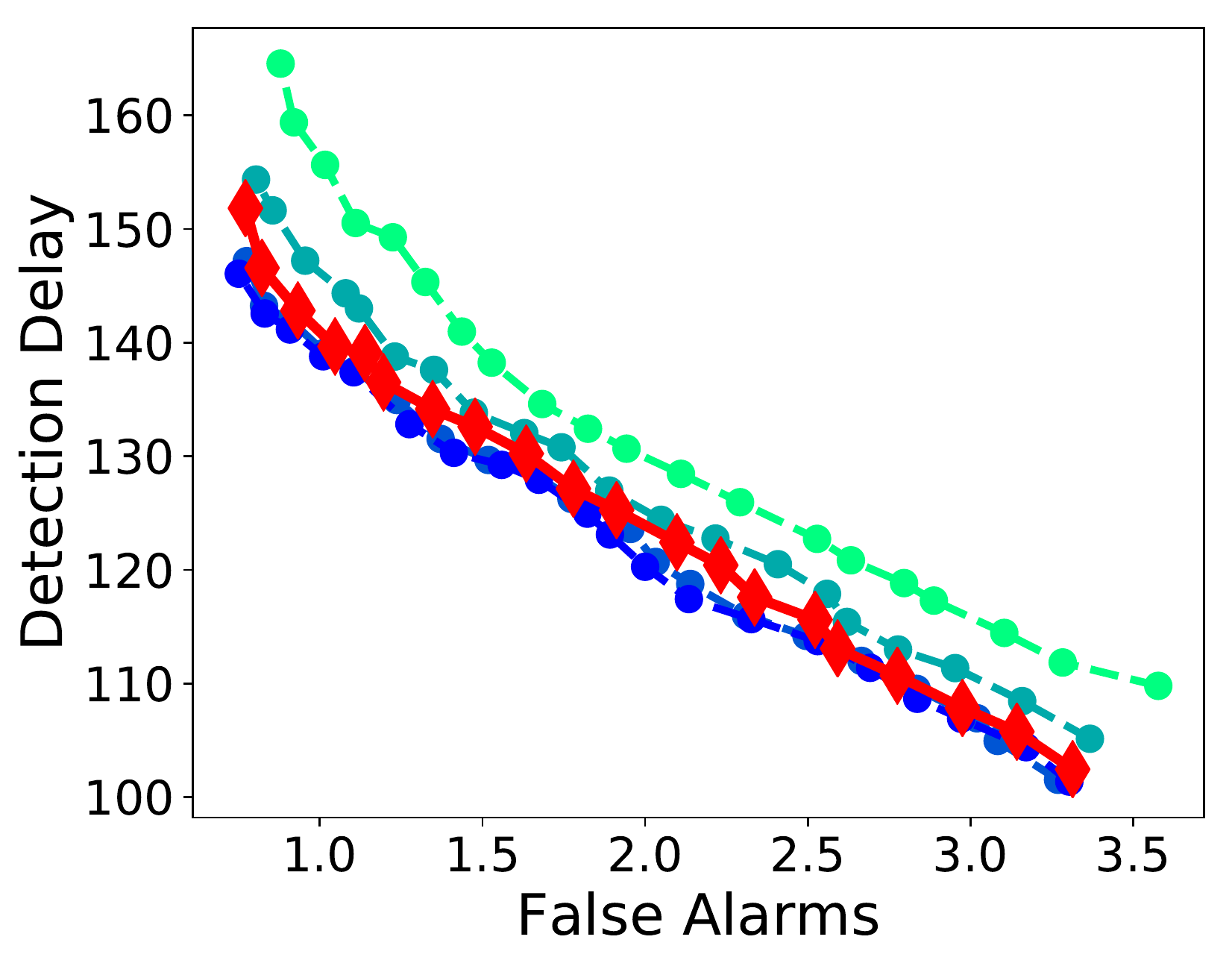} \\
\includegraphics[width = \textwidth]{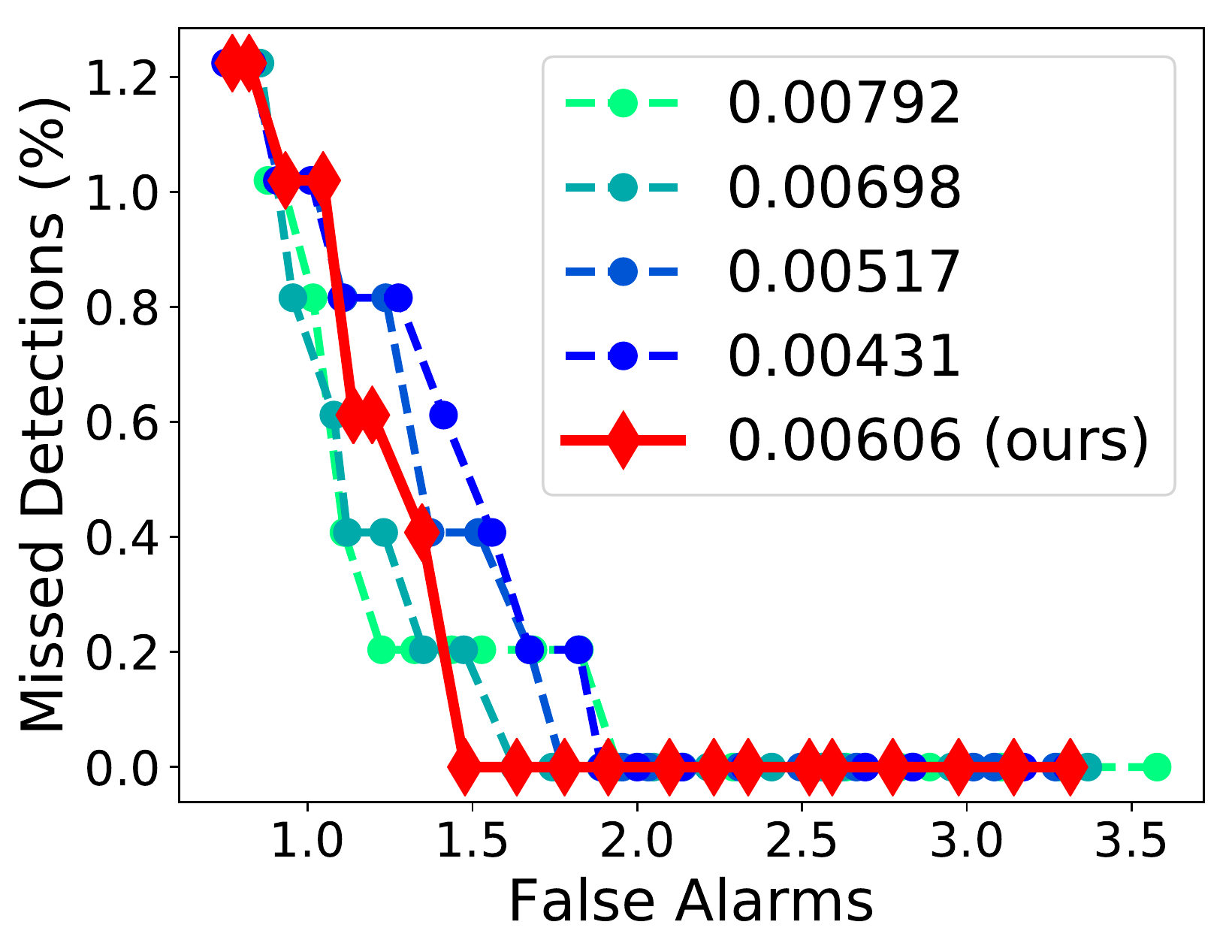}
\caption{Choice of $\updt$, for a fixed window size $\winsize$ (for $\updtp = \updtp_{\updt,\winsize}$).}
\label{subfig:lambda}
\end{subfigure}
\caption{Effect of the choice of hyperparameters in NEWMA. The thick red line indicates the proposed choices.}\label{fig:param}
\end{figure}

\begin{figure}[h!]
\centering
\begin{subfigure}{0.4\textwidth}
\includegraphics[width = \textwidth]{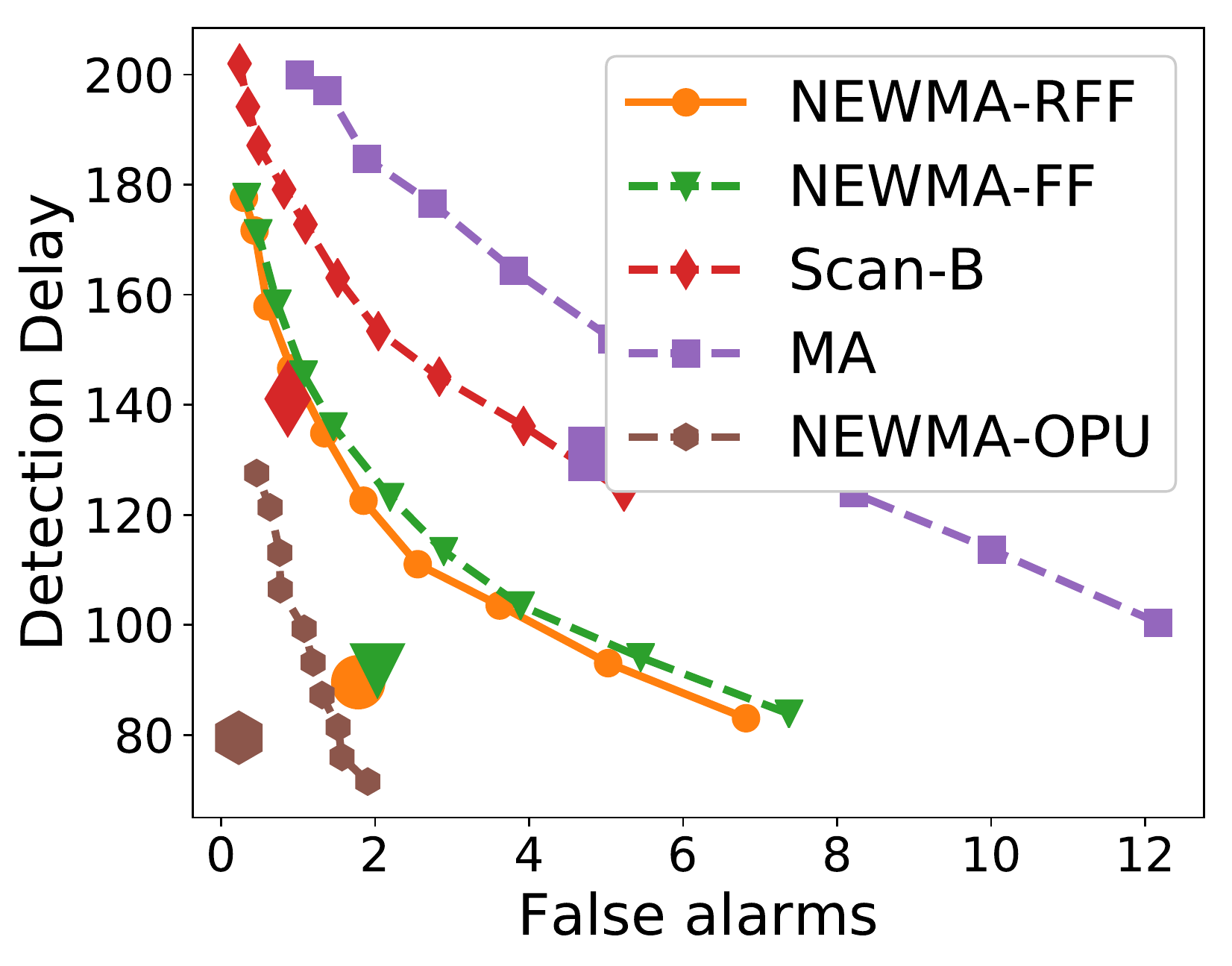} \\
\includegraphics[width = \textwidth]{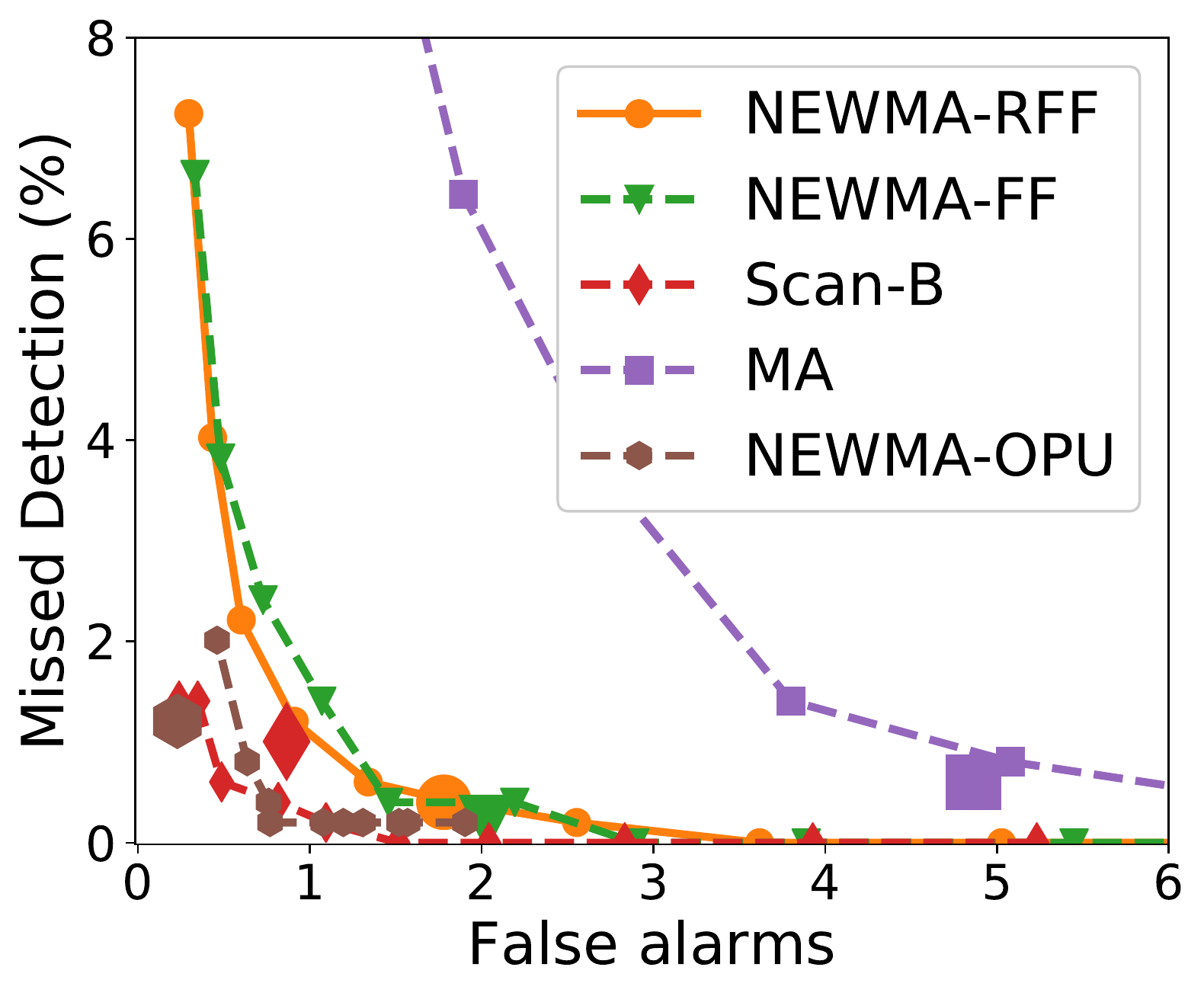}
\caption{Synthetic data.}
\label{subfig:scanB}
\end{subfigure}~
\begin{subfigure}{0.4\textwidth}
\includegraphics[width = \textwidth]{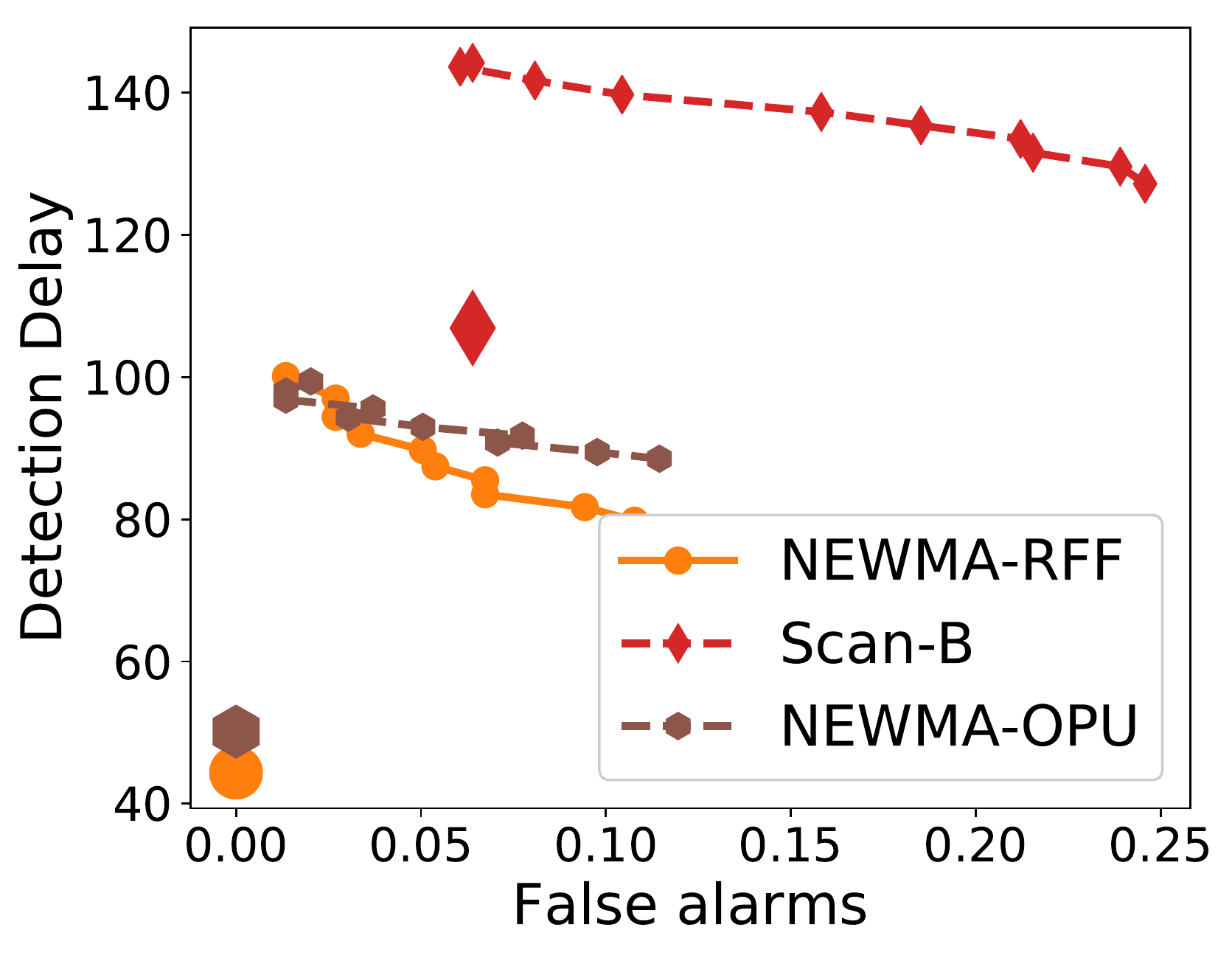} \\
\includegraphics[width = \textwidth]{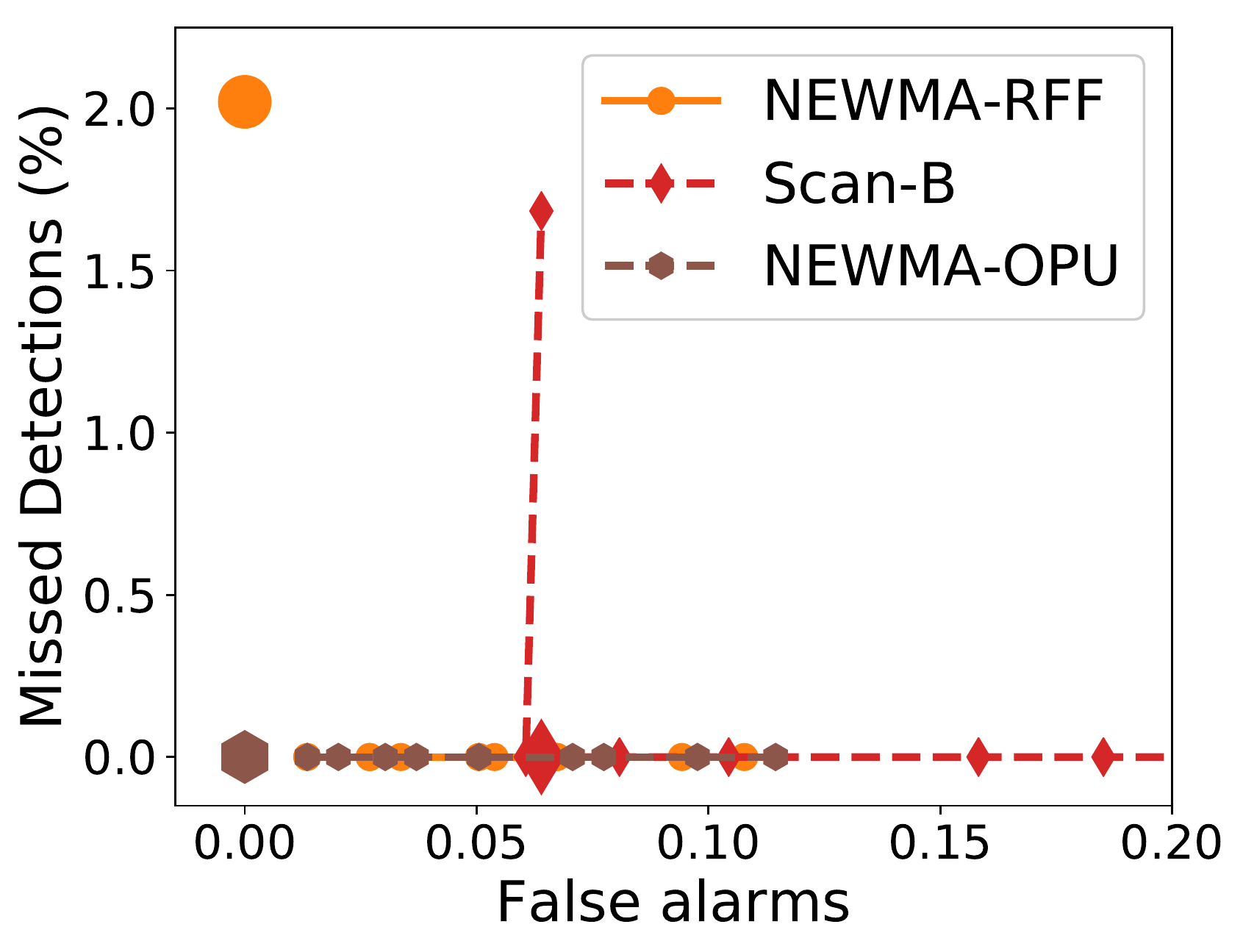}
\caption{Audio data.}
\label{subfig:bci}
\end{subfigure}
\caption{Experimental results. The solid lines corresponds to several possible choices of a fixed threshold $\tau$, while the single large dot corresponds to the performance of the adaptive threshold procedure described in Sec.~\ref{sec:adapt_thres}.}\label{fig:res}
\end{figure}

\subsection{Time of execution} 

In Fig.~\ref{fig:time} we examine the time of execution of the algorithms with respect to the dimension~$d$ of the data and window size $\winsize$. Being similar, Scan-$B$ and SW have approximately the same running time. As expected, NEWMA-FF is sublinear in the dimension, and NEWMA-OPU is almost independent of the dimension and much faster than the other approaches in high dimension. 
The results also confirm that NEWMA's complexity is independent of $\winsize$, while that of Scan-$B$ increases linearly with $\winsize$.

\subsection{Comparison of algorithms on synthetic data} 

Next we examine the detection performance of the algorithms on synthetic data. We generate the data as follows: $10^{6}$ samples are drawn from Gaussian Mixture Models (GMM) in dimension $d=100$ with $k=10$ components, and the GMM changes every $n=2000$ samples (at each change, we draw~$k$ new vector means according to a centered Gaussian distribution, $k$ new covariance matrices from an inverse-Wishart distribution, and $k$ new mixing weights from a Dirichlet distribution), \rev{resulting in 500 changes to detect in a row. We recall that these settings are more typical of online methods, where changes are detected on-the-fly, rather than offline ones, for which a high number of changes results in a high computational complexity and memory-load.} Note that \rev{the considered changes are rather complex}, with high-dimensional, multimodal, unknown distributions before and after the change, so that classical parametric change-point detection methods cannot be applied here.
For all algorithms we use a window size $\winsize = 250$.

\rev{To evaluate performance we compute false alarm rate, missed detections, and detection delay. We consider that the algorithm should be stable \emph{before} every true change and we count every detected change in the n/2 samples before it as a false alarm. We record the time until the first detected change in the n/2 samples \emph{after} every true change as detection delay, and we record a missed detection if no change is detected. This is then repeated for 500 changes in a row, and all statistics are averaged by the number of changes.}
We plot different ratios Expected Detection Delay (EDD)-to-Number of False Alarms or Missed Detections-to-Number of False Alarms (lower left corner is better), by varying a \emph{fixed} threshold $\tau$. In Fig.~\ref{fig:res}, the result of the \emph{adaptive} threshold procedure of Sec.~\ref{sec:adapt_thres} is plotted as single large dots.

\paragraph{Effects of the hyperparameters} In Fig.~\ref{fig:param}, we examine the effects of the different hyperparameters in NEWMA and our choices described in the previous sections.

In Fig.~\ref{subfig:m}, we set the forgetting factors $\updt^\star$ and $\updtp^\star$ according to Sec.~\ref{sec:optimal_param}, and vary the number of random features $m$. We compare it with our prescribed choice from Prop.~\ref{prop:MMD-RF}, $m^\star = \order{(\updt^\star + \updtp^\star)^{-2}}$ (in practice we choose an arbitrary multiplicative constant $1/4$ to reduce computation time). As expected, the performance of the detection increases with $m$, and at low $m$ the algorithm is observed to be relatively unstable. However, it is also seen that increasing $m$ beyond our choice $m^\star$ has negligible effect, so that our heuristic seems to yield the right order of magnitude for $m$. In the rest of the experiments we always choose $m = \frac{1}{4}(\updt^\star + \updtp^\star)^{-2}$, except when using the OPU, for which we choose $10$ times this value since we do not have computational restrictions in this case.

In Fig.~\ref{subfig:sigma} we examine the choice of the kernel bandwidth $\sigma$, and compare different values with the median trick that we use in practice. It is seen that the median trick yields a correct order of magnitude of about $10^2$, with all other values performing worse.

Finally, in Fig.~\ref{subfig:lambda} we vary the forgetting factor $\updt$, while keeping the window size $\winsize$ constant by choosing $\updtp =\updtp_{\updt,\winsize}$. It is seen that our prescribed choice $(\updt^\star,\updtp^\star)$ offers a balanced performance: increasing $\updt$ worsens the number of missed detections while only marginally decreasing the detection delay, and decreasing $\updt$ has the inverse effect. 

\paragraph{Comparison of algorithms} In Fig.~\ref{subfig:scanB} we compare the algorithms on synthetic data. We first observe that the adaptive threshold procedure, indicated by single large dots in the figure, is consistently better than any fixed threshold for all algorithms.

It is seen that SW performs generally poorly, confirming the superiority of Scan-$B$ as a window-based approach. NEWMA with Gaussian random features (RFF or FF) exhibits a reduced detection delay compared to Scan-$B$ but a slightly higher number of missed detections. NEWMA-OPU is seen to perform well, which may indicate that the kernel induced by the OPU is more appropriate than a Gaussian kernel on this example. 

\subsection{Real Data: voice activity detection} 

We apply our method to a Voice Activity Detection (VAD) task on audio data. We consider real environments background noise from the QUT-NOISE dataset~\cite{Dean2010} and add, every $10$s, a $3$s speech extract from the TIMIT dataset~\cite{Garofolo1993}, with $-7.5$dB Signal-to-Noise Ratio. Our goal is to detect the onset of the speech segments. We use the Short Time Fourier Transform (STFT) of the signal, \rev{a Fourier transform localized in time that is ubiquitous in audio and speech processing \cite{Cohen1995}. It allows to extract on-the-fly frequency information from the one-dimensional audio signal $s_t$, turning it into a $d=128$-dimensional time series $x_t \in \mathbb{R}^d$, where $d$ is the number of considered frequencies (usually the time axis is also dilated between $s_t$ and $x_t$). For $x_t$,} we consider a change every $1250$ samples, and $300$ changes in total. 
We take a window size $\winsize=150$.
%
We display the results in Fig.~\ref{subfig:bci}. Similar to the results on synthetic data, Scan-$B$ has a higher detection delay than NEWMA. However, it does also exhibit slightly more missed detections: we suspect that, because it uses several windows of reference in-control samples, Scan-$B$ is sensitive to highly heterogeneous data, which can be the case for audio data. In this case, the Gaussian random features are seen to perform on par with the OPU kernel. 


\section{Conclusion and outlooks}

We introduced \NEWMA{}, a new method for online change-point detection that is faster and lighter than existing model-free methods. 
The simple, key idea behind \NEWMA{} is to compare recursive averages computed with different forgetting factors on the same data, in order to extract time-varying information without keeping in memory the raw data. 

In the future, we plan to further develop the analysis of our method under the null to derive properties that do not depend on the in-control distribution, as done in~\cite{Li2015b}. \rev{Additionally, the robustness of the random generalized moments methodology to noise and missing data appears as an interesting extension.} Recent approaches for learning from random feature moments~\cite{Gribonval2017} would allow extracting more information from $\Sketch_t-\Sketch'_t$ than mere occurence of a change, without increasing the computational load. 
Another direction for our research is the study of mappings $\AnyOpNA$ for graph data, which, combined with \NEWMA{}, would allow to detect changes in large-scale social networks~\cite{Peel2014}.


\subsection*{Acknowledgment}

The authors would like to thank LightOn for the use of the OPU, Francis Bach and Sylvain Arlot for discussions on the idea behind NEWMA.

\bibliographystyle{plain}
\bibliography{library,library1,bib_dg}

\clearpage
\appendix

\section{Proofs of theorems \ref{thm:ARLuni} and \ref{thm:limit}}

We start with some elementary computations that are used throughout the rest of the proofs.
Set $0<\updtp < \updt < 1$.
For $t\geq 1$ and $i\leq t$, define $\alpha_0=(1-\updt)^t$, $\beta_0=(1-\updtp)^t$, 
$
\alpha_i = \updt(1-\updt)^{t-i} \quad \text{and}\quad \beta_i = \updtp(1-\updtp)^{t-i}
$, such that in NEWMA $\Sketch_t = \alpha_0 \Sketch_0 + \sum_{i=1}^t \alpha_i \AnyOp{\sample_i}$ and similarly for $\Sketch'_t$ and $\beta_i$.
Then, for any $1\leq t_1 < t_2\leq t$, we have
\begin{align}
\sum_{i=0}^{t_1} \alpha_i &= (1-\updt)^{t-t_1}, \label{eq:sum1t1}\\
\sum_{i=t_1}^{t_2} \alpha_i &= (1-\updt)^{t-t_2} - (1-\updt)^{t-t_1+1} \label{eq:sumt1t2}
\end{align}
and similarly for $\beta$ with $\updtp$, and
\begin{align}
\sum_{i=1}^t \alpha_i^{r_1}\beta_i^{r_2} &= \tfrac{\updt^{r_1}\updtp^{r_2}(1-A^t)}{1- A}, \text{ with }   A=(1-\updt)^{r_1}(1-\updtp)^{r_2}\, . \label{eq:sumpuiss}
\end{align}

\subsection{Mean time between false alarm (Thm. \ref{thm:ARLuni})}\label{app:ARL}
\label{sec:false-alarm}

\edit{In this section, we prove a multi-dimensional generalization of Theorem \ref{thm:ARLuni} (Theorem \ref{thm:ARL}), inspired by the approach in \cite{Fu2002}.}

Assuming all samples~$\sample_t$ are drawn \iid from a distribution $\Prob$, recall that we define the ARL under control as 
%
\begin{equation}
\label{eq:ARL}
\Tfalsealarm = \Exp\brac{\inf\set{t~|~S_t\geq \thres}}
\, .
\end{equation}
where in the case of \NEWMA{}, we have $S_t = \norm{\Sketch_t-\Sketch'_t}$. 
In this section, we derive a more tractable expression for $\Tfalsealarm$. 


Our proof strategy relies on the observation that $(\Sketch_t,\Sketch'_t)$ is a Markov chain in $\MeasSpace^2$, and thus it is possible to apply a method similar to \cite{Fu2002} for classical EWMA, with non-trivial modifications. 
%
%
We assume here that $\MeasSpace\eqdef \RR^m$, that is, $\AnyOpNA:\RR^d \to \RR^m$.
%
%
Since the stopping condition for NEWMA involves both components~$\Sketch_t$ and~$\Sketch'_t$ of the chain, we define the set $V_\thres \subset \MeasSpace^2$ as the domain in which the algorithm continues:
\begin{equation}\label{eq:continue_set}
V_\thres \eqdef \set{\vv = (\Sketch,\Sketch')~|~\norm{\Sketch-\Sketch'}< \thres} 
\, .
\end{equation}
%
%
With these notations, when we run \NEWMA{} and stop as soon as an alarm is raised, we produce a Markov chain $\vv_t \in \MeasSpace^2$ defined as: $\vv_0 = (\Sketch_0,\Sketch_0)$, and 
\[
\vv_t = \begin{cases} \paren{\begin{matrix}
(1-\updt) \vv_{t-1,1} + \updt \AnyOp{\sample_t} \\
(1-\updtp) \vv_{t-1,2} +\updtp \AnyOp{\sample_t}
\end{matrix}} &\text{ if } \vv_{t-1} \in V_\thres \, ,\\
\vv_{t-1} &\text{ otherwise.}
\end{cases}
\]
In other words, the chain is stationary as soon as an alarm is raised.
%
%
In order to state our theorem, we need to introduce a fair amount of notations. Consider the space $\MeasSpace^2 = \MeasSpace\times \MeasSpace$, equipped with the norm $\norm{(\vx,\vx')} = \norm{\vx} + \norm{\vx'}$. 
For $\varepsilon>0$, consider the ball of radius $\varepsilon^{-1}$ in $\MeasSpace^2$, $\Ball_\varepsilon = \set{\vv ~|~\norm{\vv}\leq 1/\varepsilon}$, which is compact since $\MeasSpace^2$ has finite dimension. Define $\enet_\varepsilon = \set{\vu_1,\ldots,\vu_{N_\varepsilon}} \subset \Ball_\varepsilon$ any $\varepsilon$-net of $\Ball_\varepsilon$ (we will see that its choice does not matter) such that $\vu_1 = (\Sketch_0,\Sketch_0)$, where $N_\varepsilon$ is the $\varepsilon$-covering number of $\Ball_\varepsilon$. Without lost of generality, assume they are ordered such that $\vu_1,\ldots\vu_{M_\varepsilon} \in V_\thres$ and $\vu_{M_\varepsilon+1},\ldots,\vu_{N_\varepsilon} \in V_\thres^c$ for some $M_\varepsilon$. Denote $P_\varepsilon:\MeasSpace^2 \to \enet_\varepsilon$ the projection operator onto $\enet_\varepsilon$ (\ie, that returns the $\vu_i$ closest to its input).
Define the following Markov chain $\vv_t^\varepsilon \in \enet_\varepsilon$: initialize $\vv_0^\varepsilon = (\Sketch_0,\Sketch_0) = \vu_1$, and
\[
\vv_t^\varepsilon = \begin{cases} P_\varepsilon\paren{\paren{\begin{matrix}
(1-\updt) \vv_{t-1,1}^\varepsilon + \updt\AnyOp{\sample_t} \\
(1-\updtp) \vv_{t-1,2}^\varepsilon + \updtp\AnyOp{\sample_t}
\end{matrix}}} &\text{ if } \vv_{t-1}^\varepsilon \in V_\thres \, ,\\
\vv_{t-1}^\varepsilon &\text{ otherwise.}
\end{cases}
\]
It is a \emph{projected} and \emph{bounded} version of the output of \NEWMA{}, which is stationary as soon as it gets out of $V_\thres$. 
Finally, for $1\leq i,j \leq N_\varepsilon$, define
 $
p_{ij} = \PP\paren{\vv_{t+1}^\varepsilon = \vu_j ~|~\vv_t^\varepsilon = \vu_i}
$ 
the transition probabilities of the markov chain $\vv_t^\varepsilon$.
Define $\mA = \brac{p_{ij}}_{1\leq i,j \leq M_\varepsilon}$ that corresponds to the states $\vu_i \in V_\thres$, all other states being absorbant, and $a_{ij}^{(\ell)}$ such that $\mA^\ell = \brac{a_{ij}^{(\ell)}}_{1\leq i,j \leq M_\varepsilon}$. 
Our theorem is the following.

\begin{theorem}[\textbf{Average Run Length, multidimensional case}]
\label{thm:ARL}
Assume that $\Prob$ is such that $\AnyOp{\sample} \in\RR^m$ has a density with respect to the Lebesgue measure when $x \sim \Prob$. 
Then, the quantity
 $ 
\gamma_\ell = \lim_{\varepsilon \to 0} \paren{\sum_{j=1}^{M_\varepsilon} a_{1j}^{(\ell)}}
$ 
does not depend on the choice of the nets $\enet_\varepsilon$, and the ARL of \NEWMA{} is given by
\begin{equation}\label{eq:ARL_newma}
\Tfalsealarm = \Exp\left[\inf\set{t~|~\norm{\Sketch_t-\Sketch'_t}\geq \thres}\right] = 1+\sum_{\ell\geq 1} \gamma_\ell
\, .
\end{equation}
\end{theorem}
It is then easy to check that Theorem \ref{thm:ARLuni} is an immediate consequence of Theorem \ref{thm:ARL} in the special case $m=1$.

The proof of Theorem \ref{thm:ARL} relies on the key lemma:
\begin{lemma}[\textbf{Almost sure convergence of $\vv_t^{\varepsilon}$}]
\label{lem:asConv}
For any fixed $t$, when $\varepsilon$ goes to $0$, $\vv_t^\varepsilon$ converges to $\vv_t$ almost surely.
\end{lemma}

\begin{proof}
Let us first note that, since by assumption $\AnyOp{\sample}$ has a density, is it easy to prove by recurrence on $t$ that $\vv_t$ also has a density. Therefore, for all $t$, $\PP(\vv_t \in \partial V_\thres) = 0$ (it is trivial that the boundary of $V_\thres$ has zero Lebesgue measure), and by a countable union of zero-measure sets:
\begin{equation}\label{eq:boundaryS}
\PP(\exists t ~|~ \vv_t \in \partial V_\thres) = 0
\, .
\end{equation}

Since we want to prove an almost sure convergence, we explicitly denote by $\Omega$ the set of all events such that $\forall t,~\vv_t \notin \partial V_\thres$ (which has probability $1$), and the events in $\Omega$ by $\omega$. 
A draw of a r.v. $X$ will be denoted by $X(\omega)$.

Fix any $t\geq 1$. Consider any $\omega\in\Omega$, corresponding to a draw of samples $\sample_\ell(\omega)$, $\ell=1,\ldots, t$. Remember that, by the definition of $\Omega$, $\vv_\ell(\omega) \notin \partial V_\thres$. 
Note that, when $\varepsilon$ varies, the $\vv_\ell^\varepsilon(\omega)$ change, but in a deterministic fashion. 
Our goal is to show that $\vv_t^\varepsilon(\omega) \to \vv_t(\omega)$ when $\varepsilon\to 0$.

We are going to show by induction that $\norm{\vv_\ell^\varepsilon(\omega) - \vv_\ell(\omega)} \xrightarrow[\varepsilon \to 0]{} 0$ for all $\ell=1,\ldots,t$. Since $\vv_0 = \vv_0^\varepsilon$, it is obviously true for $\ell=0$. Then, for any $\ell$, suppose that $\norm{\vv_{\ell-1}^\varepsilon(\omega) - \vv_{\ell-1}(\omega)} \xrightarrow[\varepsilon \to 0]{} 0$. 
By \eqref{eq:boundaryS} we have either $\vv_{\ell-1}(\omega) \in V_\thres$ or $\vv_{\ell-1}(\omega) \in \overline{V_\thres}^c$ since it does not belong to the boundary.
We study separately these two cases. 

\paragraph{$\vv_{\ell-1}(\omega)$ inside of $V_\thres$.}

Since by inductive hypothesis $\norm{\vv_{\ell-1}^\varepsilon(\omega) - \vv_{\ell-1}(\omega)} \xrightarrow[\varepsilon \to 0]{} 0$, and since $V_\thres$ is an open set of $\RR^{2m}$, for all $\varepsilon$ sufficiently small we have that $\vv^\varepsilon_{\ell-1}(\omega) \in V_\thres$ and the Markov chain $\vv^\varepsilon$ is updated at step $\ell$. 
Furthermore, since the radius of $\Ball_\varepsilon$ goes to $\infty$ when $\varepsilon\to 0$, for all $\varepsilon$ sufficiently small, $\vv_\ell(\omega) \in \Ball_\varepsilon$, and $\norm{P_\varepsilon(\vv_\ell(\omega)) - \vv_\ell(\omega)}\leq \varepsilon$. Hence, in that case,
\begin{align*}
&\norm{\vv_\ell^\varepsilon(\omega) - \vv_\ell(\omega)} \\
\leq&~\norm{\vv_\ell^\varepsilon(\omega) - P_\varepsilon(\vv_\ell(\omega))} + \varepsilon \notag \\
=&~\Big\lVert P_\varepsilon\paren{\paren{\begin{matrix}
(1-\updt) \vv_{\ell-1,1}^\varepsilon(\omega)+ \updt\AnyOp{\sample_\ell(\omega)} \notag \\
(1-\updtp) \vv_{\ell-1,2}^\varepsilon(\omega)+ \updtp\AnyOp{\sample_\ell(\omega)}
\end{matrix}}} \notag - P_\varepsilon\paren{\paren{\begin{matrix}
(1-\updt) \vv_{\ell-1,1}(\omega)+ \updt\AnyOp{\sample_\ell(\omega)} \\
(1-\updtp) \vv_{\ell-1,2}(\omega)+ \updtp\AnyOp{\sample_\ell(\omega)}
\end{matrix}}}\Big\rVert + \varepsilon \notag \\
\end{align*}
and since projections are contracting,
\begin{align*}
\norm{\vv_\ell^\varepsilon(\omega) - \vv_\ell(\omega)} \leq&~ \Big\lVert \paren{\begin{matrix}
(1-\updt) \vv_{\ell-1,1}^\varepsilon(\omega)+ \updt\AnyOp{\sample_\ell(\omega)} \notag \\
(1-\updtp) \vv_{\ell-1,2}^\varepsilon(\omega)+ \updtp\AnyOp{\sample_\ell(\omega)}
\end{matrix}} - \paren{\begin{matrix}
(1-\updt) \vv_{\ell-1,1}(\omega)+ \updt\AnyOp{\sample_\ell(\omega)} \notag \\
(1-\updtp) \vv_{\ell-1,2}(\omega)+ \updtp\AnyOp{\sample_\ell(\omega)}
\end{matrix}}\Big\rVert + \varepsilon \notag \\
\leq&~ (1-\updt)\norm{\vv_{\ell-1,1}^\varepsilon(\omega) - \vv_{\ell-1,1}(\omega)} + (1-\updtp)\norm{\vv_{\ell-1,2}^\varepsilon(\omega) - \vv_{\ell-1,2}(\omega)} + \varepsilon \notag \\ 
\leq&~ \norm{\vv_{\ell-1}^\varepsilon(\omega) - \vv_{\ell-1}(\omega)} + \varepsilon \, . \notag 
\end{align*}
%
Therefore $\norm{\vv_\ell^\varepsilon(\omega) - \vv_\ell(\omega)} \xrightarrow[\varepsilon \to 0]{} 0$.

\paragraph{$\vv_{\ell-1}(\omega)$ outside $\overline{V_\thres}$.}

We have $\vv_\ell(\omega) = \vv_{\ell-1}(\omega)$ by definition of the Markov chain $\vv_t$. Since $\overline{V_\thres}^c$ is an open set, by inductive hypothesis for all $\varepsilon$ sufficiently small we have $\vv_{\ell-1}^\varepsilon(\omega) \in \overline{V_\thres}^c$ and $\vv_\ell^\varepsilon(\omega) = \vv_{\ell-1}^\varepsilon(\omega)$, from which $\norm{\vv_{\ell}^\varepsilon(\omega) - \vv_\ell(\omega)} = \norm{\vv_{\ell-1}^\varepsilon(\omega) - \vv_{\ell-1}(\omega)} \xrightarrow[\varepsilon \to 0]{} 0$, which concludes the proof.
\end{proof}

We can now turn to proving the theorem itself.

\begin{proof}[Proof of Th.~\ref{thm:ARL}]

We start by a reformulation:
\begin{align*}
\Tfalsealarm &= \Exp\left[\inf\set{t~|~\norm{\Sketch_t-\Sketch'_t}\geq \thres}\right] = \Exp\paren{\inf\set{t~|~\vv_t \notin V_\thres}} \\ &= \sum_{\ell\geq 0} \PP(\inf\set{t~|~\vv_t \notin V_\thres} > \ell) = 1+\sum_{\ell\geq 1} \PP(\vv_\ell \in V_\thres)
\, ,
\end{align*}
since the first time $\vv_t$ exits $V_\thres$ is strictly greater than $\ell$ if, and only if, $\vv_\ell \in V_\thres$.
Since almost sure convergence implies weak convergence, by Lemma~\ref{lem:asConv}, we have
$
\Tfalsealarm = 1 + \sum_{\ell\geq 1} \lim_{\varepsilon \to 0} \PP(\vv_\ell^\varepsilon \in V_\thres)
$. 
Note that the convergence in $\varepsilon$ is not necessarily uniform: in general, one cannot exchange the limit operator with the infinite sum in the last display.

To conclude the proof, we just have to compute $\PP(\vv_\ell^\varepsilon \in V_\thres)$. For $1\leq i,j \leq N_\varepsilon$, recall that we denoted the transition probabilities of the Markov chain $\vv_t^\varepsilon$ by 
$
p_{ij} = \PP\paren{\vv_{t+1}^\varepsilon = \vu_j ~|~\vv_t^\varepsilon = \vu_i}
$. 
The transition matrix of this Markov chain has the form:
\[
\mP = \brac{p_{ij}}_{1\leq i,j \leq N_\varepsilon} = \brac{\begin{matrix}
\mA & \mB \\
\mathbf{0} & \Id
\end{matrix}}
\, ,
\]
where $\mA = \brac{p_{ij}}_{1\leq i,j \leq M_\varepsilon}$ corresponds to the states $\vu_i \in V_\thres$. 
Then, if we define $a_{ij}^{(\ell)}$ such that $\mA^\ell = \brac{a_{ij}^{(\ell)}}_{1\leq i,j \leq M_\varepsilon}$, it is possible to show by induction \cite{Fu2002} that:
\begin{align*}
\mP^\ell =&~ \brac{\begin{matrix}
\mA^\ell & \paren{\sum_{i=0}^{\ell-1} \mA^i}\mB \\
\mathbf{0}_{(N_\varepsilon-M_\varepsilon)\times M_\varepsilon} & \Id_{N_\varepsilon-M_\varepsilon}
\end{matrix}}
\end{align*}
and therefore
\begin{align*}
\PP(\vv_\ell^\varepsilon \in V_\thres) &= \brac{1,0,\ldots,0} \mP^\ell \paren{\begin{matrix}
\mathbf{1}_{M_\varepsilon} \\
\mathbf{0}_{N_\varepsilon-M_\varepsilon}
\end{matrix}} = \brac{1,0,\ldots,0} \mA^\ell \mathbf{1}_{M_\varepsilon} = \sum_{j=1}^{M_\varepsilon} a_{1j}^{(\ell)},
\end{align*}
which concludes the proof.
\end{proof}

\subsection{Asymptotic distribution of the statistic (Thm. \ref{thm:limit})}
\label{app:limit}

Our proof follows closely \cite{Serfling1980}, Sec. 5.5.2. with some modifications.
In the following, we let $\updtp \to 0$ with $\updt = \updtCst \updtp$ and $t \geq \frac{2}{\updtp}\log(1/\updtp)$ (which goes to $+\infty$ when $\updtp$ goes to $0$), such that $(1-\updtp)^t = \order{\updtp^2}$. 
At time $t$, we denote $\gamma_i = \beta_i-\alpha_i$, with $\alpha$ and $\beta$ defined as in the proof of Prop.~\ref{prop:decomp_sketch}. 
Note that $\alpha$ and $\beta$ also depend on $\updtp$ (and $t$), and that by Eq.~\eqref{eq:sumpuiss} we have
\begin{equation}
\frac{1}{\updtp}\sum_{i=1}^t \gamma_i^2 \xrightarrow[\updtp \to 0]{} G \eqdef \frac{(1-c)^2}{2(1+c)},
\end{equation}
and $\sum_{i=1}^t \gamma_i^q = \order{\updtp^{q-1}}$.

Define $\mu \eqdef \Exp \AnyOp{\sample}$. Using $\sum_{i=0}^t \gamma_i = 0$, at time $t$ we have
\begin{align}
\frac{1}{\updtp}\norm{\Sketch_t - \Sketch'_t}^2 &= \frac{1}{\updtp}\norm{\gamma_0 \Sketch_0 + \sum_{i=1}^t\gamma_i \AnyOp{\sample_i}}^2 \notag = \frac{1}{\updtp}\norm{ \gamma_0 (\Sketch_0-\mu) + \sum_{i=1}^t\gamma_i (\AnyOp{\sample_i}-\mu)}^2 \notag \\
&=\frac{1}{\updtp}\sum_{i, j=1}^t \gamma_i\gamma_j \kernel(\sample_i,\sample_j) + \frac{2}{\updtp} \gamma_0 \sum_{i=1}^t\gamma_i \inner{\Sketch_0-\mu,\AnyOp{\sample_i} - \mu}_\MeasSpace + \frac{1}{\updtp}\gamma_0^2 \norm{\Sketch_0 - \mu}^2\label{eq:proofConv}
\, ,
\end{align}
where $\kernel$ is a positive semi-definite kernel on $\MeasSpace$ defined by $\kernel(\sample,\sample') = \inner{\AnyOp{\sample}-\mu,\AnyOp{\sample'}-\mu}_\MeasSpace$.

The last term of Eq.~\eqref{eq:proofConv} is deterministic and goes to $0$ with $\updtp$ since $\gamma_0^2 = \order{\updtp^4}$. 
Let us now prove that the second term converges in probability to $0$. 
By Cauchy-Schwarz's and Jensen's inequalities we have
\begin{align*}
\Exp \inner{\Sketch_0-\mu,\AnyOp{\sample_i} - \mu}_\MeasSpace \inner{\Sketch_0-\mu,\AnyOp{\sample_j} - \mu}_\MeasSpace &\leq \norm{\Sketch_0 - \mu}^2\Exp \sqrt{\kernel(\sample_i,\sample_i)\kernel(\sample_j,\sample_j)}\\ 
&\leq \norm{\Sketch_0 - \mu}^2\Exp \kernel(\sample,\sample) <\infty,
\end{align*}
Hence, since $\frac{1}{\updtp}\gamma_0 =\order{\updtp}\xrightarrow[\updtp\to 0]{} 0$, it implies that $\frac{2}{\updtp} \gamma_0 \sum_{i=1}\gamma_i \inner{\Sketch_0-\mu,\AnyOp{\sample_i} - \mu}_\MeasSpace$ has a second order moment that converges to $0$, and by Markov's inequality it converges in probability to $0$.




Let us now prove that the first term in Eq.~\eqref{eq:proofConv} converges in law, and conclude with Slutsky's Lemma (Lemma~\ref{lem:conv}). 
We start by using Mercer's theorem on $\kernel$, within the ambient space $L^2(\Prob)$: we write
$
\kernel(\sample,\sample') = \sum_{\ell\geq 1} \eigen_\ell \eigenfunc_\ell(\sample)\eigenfunc_\ell(\sample')
\, ,
$
with $\eigen_\ell\geq 0$ and $\inner{\eigenfunc_\ell,\eigenfunc_{\ell'}}_{L^2(\Prob)} = 1_{\ell = \ell'}$, such that $\inner{\kernel(\sample,\cdot),\eigenfunc_\ell}_{L^2(\Prob)} = \eigen_\ell \eigenfunc_\ell(\sample)$. Note that, since $\Exp_\sample \kernel(\sample,\sample') = 0$, for any $\eigen_\ell\neq 0$ we have $\Exp\eigenfunc_\ell(\sample) = \frac{1}{\eigen_\ell}\inner{\Exp \kernel(\sample,\cdot),\eigenfunc_\ell}_{L^2(\Prob)} = 0$. Finally, we have $ \sum_{\ell\geq 1}\eigen_\ell^2 \Exp \eigenfunc_\ell^4(\sample) \leq \Exp \kernel^2(\sample,\sample) < \infty$ since $ \Exp \norm{\AnyOp{\sample}}^4 <+\infty$. 

Our goal is to show that $T_\updtp \eqdef \frac{1}{\updtp}\sum_{i,j=1}^t \gamma_i\gamma_j \kernel(\sample_i,\sample_j)$ converges in law to $Y = G \sum_{\ell\geq 1} \eigen_\ell W_\ell^2$ where $W_\ell$ are independent centered normal variable. We are going to use the characteristic function method, \ie, we are going to prove that:
\[
\forall u \in \RR,\quad \Exp e^{iuT_\updtp} \xrightarrow[\updtp \to 0]{} \Exp e^{iuY}
\, .
\]

Fix any $u\in\RR$ and $\varepsilon>0$. Our goal is to prove that, for $\updtp$ sufficiently small, we have $\abs{\Exp e^{iuT_\updtp} - \Exp e^{iuY}}\leq \varepsilon$. We decompose the bound in three parts.

\paragraph{Step 1.} For an integer $k\geq 0$, define $T_\updtp^{(k)} = \frac{1}{\updtp}\sum_{i,j=1}^t \gamma_i\gamma_j\brac{\sum_{\ell=1}^k \eigen_\ell\eigenfunc_\ell(\sample_i)\eigenfunc_\ell(\sample_j)}$. We are first going to approach $\Exp e^{iuT_\updtp}$ by $\Exp e^{iuT_\updtp^{(k)}}$ for $k$ sufficiently big. 
We write
\begin{align*}
\abs{\Exp e^{iuT_\updtp} - \Exp e^{iuT_\updtp^{(k)}}} &\leq \Exp\abs{ e^{iuT_\updtp} - e^{iuT_\updtp^{(k)}}} \\
&\leq \abs{u} \Exp\abs{T_\updtp - T_\updtp^{(k)}} \leq \abs{u} \sqrt{\Exp\paren{T_\updtp - T_\updtp^{(k)}}^2}
\, .
\end{align*}
Denote $f_k(\sample,\sample') = \kernel(\sample,\sample') - \sum_{\ell=1}^k\eigen_\ell \eigenfunc_\ell(\sample)\eigenfunc_{\ell}(\sample') = \sum_{\ell\geq k+1}\eigen_\ell \eigenfunc_\ell(\sample)\eigenfunc_{\ell}(\sample')$, such that $T_\updtp - T_\updtp^{(k)} = \frac{1}{\eta}\sum_{i,j=1}^t \gamma_i\gamma_j f_k(\sample_i,\sample_j)$. We have $\Exp \brac{f_k(\sample_{1},\sample'_{1})f_k(\sample_{2},\sample'_{2})} \neq 0$ if and only if both $\sample_1 = \sample_2$ and $\sample'_1 = \sample'_2$ (or permuted since $f_k$ is symmetric), and we have
\begin{align*}
&\Exp_{\sample,\sample'} f_k(\sample,\sample')^2 = \sum_{\ell \geq k+1} \eigen_\ell^2 \paren{\Exp \eigenfunc_\ell^2(\sample)}^2 = \sum_{\ell \geq k+1} \eigen_\ell^2 \\
&\Exp f_k(\sample,\sample)^2 = \sum_{\ell \geq k+1} \eigen_\ell^2 \Exp \eigenfunc_\ell^4(\sample)
\, ,
\end{align*}
where the last expression is summable since $\Exp \kernel^2(\sample,\sample)< \infty$.
Then we have
\begin{align*}
\Exp&\paren{T_\updtp - T_\updtp^{(k)}}^2 \\
&= \frac{1}{\updtp^2}\Exp\brac{\sum_{i_1, j_1=1}^t\sum_{i_2, j_2=1}^t\gamma_{i_1}\gamma_{j_1}\gamma_{i_2}\gamma_{j_2} f_k(\sample_{i_1},\sample_{j_1}) f_k(\sample_{i_2},\sample_{j_2})} \\
&\leq\frac{2}{\updtp^2}\Exp\brac{\sum_{i, j=1}^t\gamma_{i}^2\gamma_{j}^2 f_k(\sample_{i},\sample_{j})^2} \leq 2\paren{\frac{1}{\updtp}\sum_{i=1}^t \gamma_i^2}^2 \max\paren{\Exp_{\sample,\sample'} f_k^2(\sample,\sample'),\Exp f_k^2(\sample,\sample)} \\
&\leq C\paren{\sum_{\ell\geq k+1}\eigen_\ell^2\max\paren{1,\Exp\eigenfunc_\ell^4(\sample)}}
\, ,
\end{align*}
for some constant $C$, since $\sum_{i=1}^t\gamma_i^2 = \order{\updtp}$. Hence for $k$ sufficiently big we have:
\begin{equation}\label{eq:limitproof1}
\forall \updtp\in (0,~1),~\abs{\Exp e^{iuT_\updtp} - \Exp e^{iuT_\updtp^{(k)}}} \leq \frac{\varepsilon}{3}
\, .
\end{equation}

\paragraph{Step 2.} Let us now temporarily consider a fixed $k$, and prove that $T_\updtp^{(k)}$ converges in law to $Y^{(k)} = G \sum_{\ell=1}^k \eigen_\ell W_\ell^2$. We write
\begin{align*}
T_\updtp^{(k)} &= \frac{1}{\updtp}\sum_{i,j=1}^t \gamma_i\gamma_j\brac{\sum_{\ell=1}^k \eigen_\ell\eigenfunc_\ell(\sample_i)\eigenfunc_\ell(\sample_j)} = \sum_{\ell=1}^k \eigen_\ell\paren{\frac{1}{\sqrt{\updtp}}\sum_{i=1}^t \gamma_i\eigenfunc_{\ell}(\sample_i)}^2
\, .
\end{align*}
We now use Lindeberg's theorem (Th.~\ref{thm:lindeberg}) on the random vectors
$
X^{(i,t)} = \paren{\frac{1}{\sqrt{\updtp}} \gamma_i\eigenfunc_{\ell}(\sample_i)}_{\ell=1}^k
$. 
They are centered and their covariance is such that
\[
\sum_{i=1}^t Cov(X^{(i,t)}) = \paren{\frac{1}{\updtp}\sum_{i=1}^t \gamma_i^2}\Id \xrightarrow[\updtp\to 0]{} G\cdot \Id
\, .
\]
We now check Lindeberg's condition \eqref{eq:lindebergcond}. By Cauchy-Schwartz and Markov's inequality, for all $\delta>0$ we have
\begin{align*}
\Exp\brac{\norm{X^{(i,t)}}^2 \kroneck{\norm{X^{(i,t)}}\geq \delta}} &\leq \sqrt{\Exp \norm{ X^{(i,t)}}^4}\cdot \PP\brac{\norm{X^{(i,t)}}\geq \delta} \\
&\leq \sqrt{\Exp\norm{X^{(i,t)}}^4}\cdot \delta^{-2} \Exp\norm{X^{(i,t)}}^2 \leq C \delta^{-2} \frac{\gamma_i^4}{\updtp^2}
\, ,
\end{align*}
where $C$ is a constant, since $\eigenfunc_\ell(\sample)$ has finite second and fourth order moment. Using the fact that $\sum_{i=1}^t \frac{\gamma_i^4}{\updtp^2} = \order{\updtp}$, Lindeberg's condition is satisfied. Hence, applying theorem \ref{thm:lindeberg}, $\sum_{i=1}^t X^{(i,t)}$ converges in law to $\mathcal{N}(0, G\cdot \Id)$, and $T_\updtp^{(k)}$ converges in law to $Y^{(k)}$. Hence for a sufficiently small $\updtp$
\begin{equation}\label{eq:prooflimit2}
\abs{\Exp e^{iu T^{(k)}_\updtp} - \Exp e^{iuY^{(k)}}} \leq \frac{\varepsilon}{3}
\, .
\end{equation}

\paragraph{Step 3.} Finally, similar to Step 1 we have
\begin{align*}
\abs{\Exp e^{iuY^{(k)}} - \Exp e^{iuY}} &\leq \Exp\abs{ e^{iuY^{(k)}} - e^{iuY}} \leq \abs{u} \Exp\abs{Y^{(k)} - Y} \\
&\leq \abs{u} \sqrt{\Exp\paren{Y^{(k)} - Y}^2} \leq \abs{u} \Big(\sum_{\ell\geq k+1} \eigen_\ell\Big)^2 \max(1,\Exp W^4)
\end{align*}
where $W\sim \mathcal{N}(0,1)$,and therefore for a sufficiently big $k$
\begin{equation}
\label{eq:prooflimit3}
\abs{\Exp e^{iuY^{(k)}} - \Exp e^{iuY}} \leq \frac{\varepsilon}{3}
\, .
\end{equation}

To conclude, we fix $k$ large enough such that Eq.~\eqref{eq:limitproof1} and~\eqref{eq:prooflimit3} are satisfied, then~$\updtp$ small enough and Eq.~\eqref{eq:prooflimit2} is satisfied, which concludes the proof.
\qed


%

\section{Technical proofs}

\subsection{Proof of Prop.~\ref{prop:decomp_sketch}}
\label{sec:proof-decomp-sketch}

Recall that we defined 
\[
\winsize =\left\lceil \frac{\log\paren{\updt/\updtp}}{\log\paren{(1-\updtp)/(1-\updt)}}\right\rceil
\, .
\]
Let $t>\winsize\geq 1$.
By construction of definition of $\Sketch_t$ and $\Sketch_t'$ and by definition $\alpha_i,\beta_i$, we have
\[
\begin{cases}
\Sketch_t &= \alpha_0\Sketch_0 + \sum_{i=1}^t \alpha_i \AnyOp{\sample_i} \\
\Sketch_t' &= \beta_0\Sketch_0 + \sum_{i=1}^t \beta_i \AnyOp{\sample_i} \, .
\end{cases}
\]
%

A straightforward computation yields that $t-\winsize$ is the ``shifting'' point for the weight coefficients. 
Namely, for $i=1,\ldots,t-\winsize$, we have $\alpha_i\geq \beta_i$, and for $i=t-\winsize+1,\ldots,t$, we have $\beta_i\geq \alpha_i$. 
According to Eq.~\eqref{eq:sum1t1} and Eq.~\eqref{eq:sumt1t2}, we have $\sum_{i=0}^{t-\winsize} (\alpha_i - \beta_i) = \sum_{i=t-\winsize+1}^t (\beta_i - \alpha_i) = C$, where $C = (1-\updtp)^\winsize - (1-\updt)^\winsize$. 
Hence, if we define $a_i \eqdef (\alpha_i-\beta_i)/C$ for $i=0,\ldots,t-\winsize$ and $b_i \eqdef (\beta_i-\alpha_i)/C$ for $i=t-\winsize+1,\ldots,t$, we have
\begin{align*}
\Sketch_t - \Sketch_t' &= \sum_{i=1}^t (\alpha_i - \beta_i) \AnyOp{\sample_i} + (\alpha_0-\beta_0)\Sketch_0 \\
&= \sum_{i=t-\winsize+1}^t (\alpha_i - \beta_i) \AnyOp{\sample_i} - \paren{\sum_{i=1}^{t-\winsize} (\beta_i - \alpha_i) \AnyOp{\sample_i} + (\beta_0-\alpha_0)\Sketch_0} \\
\frac{\Sketch_t - \Sketch_t}{C} &= \sum_{i=t - \winsize+1}^t a_i \AnyOp{\sample_i} - b_0\Sketch_0 - \sum_{i=1}^{t-\winsize} b_i \AnyOp{\sample_i}
\, .
\end{align*}
\qed

\subsection{Proof of Prop.~\ref{prop:pointwise_detection}}\label{app:proof_pointwise}

Prop.~\ref{prop:pointwise_detection} is a direct consequence from the following, more general result.

\begin{lemma}[\textbf{Concentration of the detection statistic}]
\label{lem:conc}
Suppose that $M = \sup_{\sample\in\RR^d} \norm{\AnyOp{\sample}} < +\infty$. At time $t$, assume that the last $B_1$ samples are drawn according to $\Prob'$, and that the $B_2$ samples that came immediately before were drawn from $\Prob$ (earlier samples can be arbitrary), and that $\winsize_1+\winsize_2 \geq \winsize$ for simplicity. Then, with probability at least $1-\pFail$ on the samples, we have
\begin{equation}\label{eq:Rade}
\Big|\norm{\Sketch_t - \Sketch'_t} - C\distphi(\Prob',\Prob)\Big| \leq  \Efund + \Einit + \Ewin
\, ,
\end{equation}
with 
\begin{align*}
\Efund &=  4\sqrt{2}M\sqrt{\log\tfrac{2}{\pFail}} \sqrt{\varphi(\updt,\updt) + \varphi(\updtp,\updtp) - 2\varphi(\updt,\updtp)} 
\, ,\\ 
\Einit &= \paren{(1-\updtp)^t - (1-\updt)^t}\norm{\Sketch_0  - \Exp_\Prob \AnyOp{\sample}} 
\, ,\\
\Ewin &=  2M\Big( f(\updtp) - f(\updt) + \abs{g(\updtp)- g(\updt)}\Big) 
\end{align*}
where $\varphi(a,b)=\tfrac{ab \paren{1- (1-a)^{t} (1-b)^{t}}}{a + b - ab}$, $f(a) = (1-a)^{\winsize_1+\winsize_2} - (1-a)^t$, and $g(a) = (1-a)^{\underline{\winsize}} - (1-a)^{\overline{\winsize}}$, with $\underline{\winsize} = \min(\winsize,\winsize_1)$ and $\overline{\winsize} = \max(\winsize,\winsize_1)$.

\end{lemma}

\begin{proof}
We have seen that, ideally, the last $\winsize$ samples are drawn from $\Prob'$, and all the samples that came before are drawn from $\Prob$. 
Let us call $I$ the time interval of samples that are not drawn from the ``correct'' distribution:
\[
I\eqdef\llbracket 1,~t-\winsize_1-\winsize_2 \rrbracket \cup\llbracket t-\overline{\winsize}+1,~t-\underline{\winsize}\rrbracket
\, .
\]

Let us introduce ``ghost samples'' $y_1,\ldots,y_t$ drawn from the ``correct'' distributions, \ie such that $y_1,\ldots,y_{t-\winsize}\simiid \Prob$, $y_{t-\winsize+1},\ldots,y_t \simiid \Prob'$, and such that $y_i = x_i$ for $i\notin I$. 
The idea of the proof is to introduce the analogous of $\norm{\Sketch_t-\Sketch'_t}$ for the ghost samples in the left-hand side of Eq.~\eqref{eq:Rade}, to use the triangle inequality, and then to bound the resulting error terms.
Thus, with the help of Prop.~\ref{prop:decomp_sketch}, introducing $\gamma_i \eqdef \alpha_i - \beta_i$, we first write
\begin{align*}
&\Big|\norm{\Sketch_t-\Sketch'_t} - C\distphi(\Prob,\Prob')\Big| \\ 
	&= \abs{\norm{\gamma_0\Sketch_0+ \sum_{i=1}^{t}\gamma_i \AnyOp{\sample_i}}  - C\norm{\Exp_\Prob \AnyOp{y} - \Exp_{\Prob'} \AnyOp{y}}}  \\
&\leq \abs{\norm{\gamma_0\Sketch_0+\sum_{i=1}^{t}\gamma_i \AnyOp{y_i}} - C\norm{\Exp_\Prob \AnyOp{y} - \Exp_{\Prob'} \AnyOp{y}}} + \abs{\norm{\gamma_0\Sketch_0+\sum_{i=1}^{t}\gamma_i \AnyOp{y_i}} - \norm{\gamma_0\Sketch_0+\sum_{i=1}^{t}\gamma_i \AnyOp{x_i}}}  \\
&\eqdef \text{(I)} + \text{(II)}  
\, .
\end{align*}
since $\abs{x-y}\leq \abs{z-x}+\abs{z-y}$.

We first show that (II) is upper bounded by $\Ewin$.
Since $y_i=x_i$ for any $i\notin I$, 
\[
\text{(II)} \leq 2M\sum_{i\in I} \abs{\gamma_i} \notag
\, .
\]
%
By definition of the integer interval $I$,
\begin{align*}
\sum_{i\in I} \abs{\gamma_i} &=\sum_{i= 1}^{t-\winsize_1-\winsize_2} \abs{\gamma_i} + \sum_{i=t-\overline{\winsize}+1}^{t-\underline{\winsize}} \abs{\gamma_i} =\abs{\sum_{i= 1}^{t-\winsize_1-\winsize_2} \gamma_i} + \abs{\sum_{i=t-\overline{\winsize}+1}^{t-\underline{\winsize}} \gamma_i}
\, ,
\end{align*}
since $\gamma_i$ has constant sign in the considered intervals. 
Using Eq.~\eqref{eq:sum1t1} and~\eqref{eq:sumt1t2}, we obtain the desired expression for $\Ewin$. 

We now prove that (I) is upper bounded by $\Einit + \Efund$.
By the triangle inequality and the definition of $a_i$ and $b_i$,
\begin{align}
\text{(I)} &\leq \Bigg\lVert(\beta_0 - \alpha_0)\Sketch_0 + \sum_{i=1}^{t-\winsize}(\beta_i - \alpha_i)\AnyOp{y_i} - C\Exp_\Prob \AnyOp{y} -\paren{\sum_{i=t-\winsize+1}^{t}(\alpha_i-\beta_i)\AnyOp{y_i} - C\Exp_{\Prob'} \AnyOp{y}}\Bigg\rVert \notag \\
&\leq Cb_0\norm{\Sketch_0-\Exp_\Prob\AnyOp{y}} + C\norm{\sum_{i=1}^{t-\winsize}b_i (\AnyOp{y_i} - \Exp_\Prob \AnyOp{y})} + C\norm{\sum_{i=t-\winsize+1}^{t}a_i (\AnyOp{y_i} - \Exp_{\Prob'} \AnyOp{y})} \notag 
\end{align}
since $\sum_{i=0}^{t-\winsize} (\beta_i - \alpha_i) = \sum_{i=t-\winsize+1}^t (\alpha_i - \beta_i) = C$.

We now apply McDiarmid's inequality (Lemma~\ref{lem:mcdiarmid}) to bound the right-hand side of the last display with high probability. 
Define $\Delta:(\RR^d)^{t-\winsize}\to\RR$ by
\[
\Delta(y_1,\ldots,y_{t-\winsize}) = \norm{\sum_{i=1}^{t-\winsize}b_i (\AnyOp{y_i} - \Exp_\Prob \AnyOp{y})}
\, .
\]
This function satisfies the bounded difference property, that is, 
%
%
\begin{equation*}
\abs{ \Delta(y_1,\ldots,y_i,\ldots,y_{t-\winsize}) - \Delta(y_1,\ldots,y'_i,\ldots,y_{t-\winsize})}
\leq  2M a_i
\, .
\end{equation*}
We then apply Lemma~\ref{lem:mcdiarmid} with $f=\Delta$ and $c_i=2Ma_i$ to obtain
%
\[
\PP\paren{\Delta \geq \Exp \Delta + \varepsilon}\leq \exp\paren{-\frac{\varepsilon^2}{4M^2\paren{\sum_{i=1}^{t-\winsize} b_i^2}}}
\, .
\]
%
We now bound $\Exp\Delta$ by a symmetrization argument.
Let us introduce the random variables $y_i'$ that have the same law as the $y_i$ and are independent from the $y_i$, and the $\sigma_i$, Rademacher random variables independent from both $y_i$ and $y_i'$.
We write
\begin{align*}
\Exp \norm{\sum_{i=1}^{t-\winsize}b_i (\AnyOp{y_i} - \Exp \AnyOp{y})} &=\Exp_{y,y'} \norm{\sum_{i=1}^{t-\winsize}b_i \paren{\AnyOp{y_i} - \AnyOp{y_i'}} } = \Exp_{y,y',\sigma} \norm{\sum_{i=1}^{t-\winsize}b_i\sigma_i \paren{\AnyOp{y_i} - \AnyOp{y_i'}} }\\
&\leq 2\Exp_{y}\Exp_{\sigma}\norm{\sum_{i=1}^{t-\winsize} b_i \sigma_i \AnyOp{y_i}} = 2\sqrt{ \Exp_{y}\Exp_{\sigma}\sum_{i,j=1}^{t-\winsize} b_i b_j \sigma_i \sigma_j \inner{\AnyOp{y_i}, \AnyOp{y_j}}_\MeasSpace}  \\
&\leq 2M\sqrt{\sum_{i=1}^{t-\winsize} b_i^2}
\end{align*}

By applying the same reasoning to $\Delta' \eqdef \norm{\sum_{i=t-\winsize+1}^{t}a_i (\AnyOp{y_i} - \Exp_{\Prob'} \AnyOp{y})}$ and a union bound, we obtain that, with probability at least $1-\pFail$,
\begin{equation*}
\text{(I)} \leq \Einit + 2MC\paren{1+\sqrt{\log\tfrac{2}{\pFail}}}\paren{\Big(\sum_{i=t-\winsize+1}^{t} a_i^2\Big)^\frac12 + \Big(\sum_{i=1}^{t-\winsize} b_i^2\Big)^\frac12}
\end{equation*}
%
%
Since
\begin{align*}
\sqrt{ \sum_{i=t-\winsize+1}^{t} a_i^2} + \sqrt{\sum_{i=1}^{t-\winsize} b_i^2} \leq&~ \sqrt{2}\sqrt{\sum_{i=t-\winsize+1}^{t} a_i^2 +\sum_{i=1}^{t-\winsize} b_i^2} = \frac{\sqrt{2}}{C} \sqrt{\sum_{i=1}^t \left(\alpha_i^2 + \beta_i^2 - 2\alpha_i\beta_i\right)}
\, ,
\end{align*}
%
we recover the expression of $\Efund$ with the help of Eq.~\eqref{eq:sumpuiss}.
Therefore, we showed that, with probability at least $1-\rho$,
\[
\abs{\norm{\Sketch_t-\Sketch_t'} - C\distphi(\Prob,\Prob')} \leq \Efund + \Einit + \Ewin
\, .
\]

\end{proof}

\subsection{Proof of Prop.~\ref{prop:MMD-RF}}\label{app:proof_MMD}

The proof is a combination of Prop.~\ref{prop:pointwise_detection} and the following lemma, which is a simple consequence of Hoeffding's inequality.

\begin{lemma}[\textbf{Concentration of $\distphi$}]
Define $\AnyOpNA$ as \eqref{eq:defRF} and let $\pFail\in (0,1)$. 
For any distributions $\Prob,\Prob'$, with probability at least $1-\pFail$ on the $\freq_j$'s, we have
\[
\distphi(\Prob, \Prob')^2 \geq \MMD(\Prob,\Prob')^2 - \frac{2\sqrt{2}M^2}{\sqrt{m}}\sqrt{\log\frac{1}{\pFail}}
\, .
\]
\end{lemma}

\begin{proof}
%
%
One can show \cite{Sriperumbudur2010} that the MMD can also be expressed as
\[
\MMD(\Prob,\Prob')^2 = \int \abs{\rfeat(\Prob) - \rfeat(\Prob')}^2 \diff \freqdist(\freq)
\, .
\]
where $\rfeat(\Prob) \eqdef \int \rfeat(x) \diff \Prob(x)$. 
By the definition of $\AnyOp{\sample} = \frac{1}{\sqrt{m}}\paren{\feat_{\freq_j}(\sample)}_{j=1}^m$,
\[
\distphi(\Prob, \Prob')^2 = \frac{1}{m}\sum_{j=1}^m \abs{\feat_{\freq_j}(\Prob) - \feat_{\freq_j}(\Prob')}^2
\, .
\]
Since $\sup_{\sample,\freq} \abs{\rfeat(\sample)} \leq M$, we deduce that $\abs{\feat_{\freq_j}(\Prob) - \feat_{\freq_j}(\Prob')}^2 \leq 4M^2$ we can apply Hoeffding's inequality (Lemma~\ref{lem:hoeffding}) to $\distphi$.
Thus, with probability $1-\pFail$, it holds that
\[
\MMD(\Prob,\Prob')^2 - \distphi(\Prob, \Prob')^2 \leq \frac{2\sqrt{2}M^2}{\sqrt{m}}\sqrt{\log\frac{1}{\pFail}}
\, .
\]
\end{proof}

\section{Third-party technical results}
\label{sec:third-party}

We gather here some existing results that were used in the proofs.

\begin{lemma}[\textbf{Hoeffding's inequality (\cite{Bou_Lug_Mas:2013}, Th.~2.8)}]
	\label{lem:hoeffding}
	Let $X_i$ be independent, bounded random variables such that $X_i \in [a_i,~b_i]$ a.s. 
	Then, for any $t>0$, 
	\[
	\PP\paren{\frac{1}{n}\sum_{i=1}^n X_i - \Exp\paren{\frac{1}{n}\sum_{i=1}^n X_i}\geq t}\leq e^{-\frac{2n^2 t^2}{\sum_{i=1}^n (a_i -b_i)^2}}
	\, .
	\]
\end{lemma}

\begin{lemma}[\textbf{McDiarmid's inequality (\cite{Bou_Lug_Mas:2013}, Th.~6.2)}]
	\label{lem:mcdiarmid}
	Let $f:E^n\to \RR$ be a measurable function that satisfies the \emph{bounded difference} property, that is, there exist positive numbers $c_1,\ldots,c_n$ such that
	\begin{equation*}
	\sup_{x_1,\ldots,x_n,x'_i\in E} \abs{f(x_1,\ldots,x_i,\ldots,x_n) - f(x_1,\ldots,x'_i,\ldots,x_n)} \leq c_i
	\, .
	\end{equation*}
	Let $X_1,\ldots,X_n$ be independent random variables with values in $E$ and set $Z = f(X_1,...,X_n)$. 
	Then 
	\[
	\PP\paren{Z - \Exp Z \geq t} \leq \exp\paren{-\frac{2t^2}{\sum_{i=1}^n c_i}}
	\, .
	\]
\end{lemma}

\begin{theorem}[\textbf{Multivariate Lindeberg's Theorem (\cite{VanderVaart1998}, Th.~2.27)}]
	\label{thm:lindeberg}
	For each $n$, let $X^{(i,n)}$, $1\leq i \leq n$, be independent, $\RR^d$-valued random vectors with zero mean and covariance $\Sigma^{(i,n)}$ such that $\sum_{i=1}^n \Sigma^{(i,n)} \to \Sigma$ when $n \to \infty$ (for the Frobenius norm), and Lindeberg's condition is satisfied:
	\begin{equation}\label{eq:lindebergcond}
	\forall \varepsilon > 0,~\sum_{i=1}^n \Exp\brac{\norm{X^{(i,n)}}^2 I\set{\norm{X^{(i,n)}}>\varepsilon}} \to 0 
	\, ,
	\end{equation}
	when $n \to \infty$. Then $S_n = \sum_{i=1}^n X^{(i,n)}$ converges in law toward a centered Gaussian with covariance $\Sigma$.
\end{theorem}

\begin{lemma}[\textbf{Slutsky's Lemma (\cite{VanderVaart1998}, Th. 2.7)}]
	\label{lem:conv}
	Let $X_n$, $X$, $Y_n$ be random variables. If $X_n$ converges \emph{in law} to $X$ and $Y_n$ converges \emph{in probability} to a constant $c$, then $(X_n,Y_n)$ converges in law to $(X,c)$.
\end{lemma}

\end{document}